\documentclass[letterpaper]{article} 
\usepackage{aaai24}
\usepackage{times}
\usepackage{helvet}
\usepackage{courier}
\usepackage[hyphens]{url}
\usepackage{graphicx}
\usepackage{natbib}
\usepackage{caption}
\newcommand{\Hide}[1]{}
\usepackage{enumitem}
\usepackage{xcolor}

\usepackage{algorithm}
\usepackage{algorithmic}
\usepackage{makecell}
\usepackage{multirow}
\usepackage{amsmath}
\usepackage{amssymb}
\usepackage{amsthm}
\newtheorem{theorem}{Theorem}
\newtheorem{lemma}{Lemma}
\newtheorem{definition}{Definition}
\newcommand{\vpara}[1]{\vspace{0.05in}\noindent\textbf{#1 }}

\usepackage{newfloat}
\usepackage{listings}
\usepackage{arydshln}
\DeclareCaptionStyle{ruled}{labelfont=normalfont,labelsep=colon,strut=off}
\floatstyle{ruled}
\newfloat{listing}{tb}{lst}{}
\floatname{listing}{Listing}

\title{Towards Fair Graph Federated Learning via Incentive Mechanisms}

\author{
    Chenglu Pan\textsuperscript{\rm 1,2}\equalcontrib\thanks{This work was done when the author was a visiting student at Fudan University.}, Jiarong Xu\equalcontrib\textsuperscript{\rm 2}\thanks{Corresponding author.}, Yue Yu\textsuperscript{\rm 2}, Ziqi Yang \textsuperscript{\rm 1}, Qingbiao Wu\textsuperscript{\rm 1},\\ Chunping Wang\textsuperscript{\rm 3}, Lei Chen\textsuperscript{\rm 3}, Yang Yang\textsuperscript{\rm 1}
}
\affiliations{

    \textsuperscript{\rm 1}Zhejiang University,
    \textsuperscript{\rm 2}Fudan University,
    \textsuperscript{\rm 3}FinVolution Group\\
    \{chenglupan,yangziqi,qbwu,yangya\}@zju.edu.cn, jiarongxu@fudan.edu.cn,
    yuyue22@m.fudan.edu.cn,\\
    \{wangchunping02,chenlei04\}@xinye.com
    
}

\usepackage{bibentry}

\begin{document}

\maketitle

\begin{abstract}
Graph federated learning (FL) has emerged as a pivotal paradigm enabling  multiple agents to collaboratively train a graph model while preserving local data privacy. Yet, current efforts overlook a key issue: agents are self-interested and would hesitant to share data without fair and satisfactory incentives. This paper is the first endeavor to address this issue by studying the incentive mechanism for graph federated learning. We identify a unique phenomenon in graph federated learning: the presence of agents posing potential harm to the federation and agents contributing with delays. This stands in contrast to previous FL incentive mechanisms that assume all agents contribute positively and in a timely manner. 
In view of this, this paper presents a novel incentive mechanism  tailored for fair graph federated learning, integrating incentives derived from both model gradient and payoff. To achieve this, we first introduce an agent valuation function aimed at quantifying agent contributions through the introduction of two criteria: gradient alignment and graph diversity. Moreover, due to the high heterogeneity in graph federated learning, striking a balance between accuracy and fairness becomes particularly crucial. We introduce motif prototypes to enhance accuracy, communicated between the server and agents, enhancing global model aggregation and aiding agents in local model optimization. Extensive experiments show that our model achieves the best trade-off between accuracy and the fairness of model gradient, as well as superior payoff fairness.

\end{abstract}

\vspace{-0.3in}
\section{Introduction}
\vspace{-0.01in}
Graph data is ubiquitous, exhibiting diverse and generic connectivity patterns, yet a notable portion is distributed or isolated among distinct agents (\emph{e.g.}, companies, research institutions). Unlocking the potential of this ``closed graph data'' represents a hidden gold mine in the era of big data. Recent advances of graph federated learning offer opportunities for collaboration among multiple agents without compromising data privacy, with each agent conducting local model training and sharing their updates/gradients with a global model on a server~\cite{mcmahan2017communication, karimireddy2020scaffold, li2020federated, zhao2018federated,xie2021federated, zhang2021subgraph, xie2023federated, gu2023dynamic, Zhang2021FASTGNNAT}. For instance, financial companies with their own transaction networks can join forces to enhance fraud detection model by participating in a graph federation.

However, in reality, the agents are self-interested and may not be cooperative if all agents receive the same model while their contributions differ.  This implies that to achieve a competitive global graph model, there is a strong need to establish a fair graph federated learning framework that incentives agents to provide high-quality information. 
In view of this, several studies have explored incentive mechanisms for FL, with a primary focus on image domain~\cite{xu2020reputation, xu2021gradient, Deng2021FAIRQF, Gao2021FIFLAF}. 

These previous works almost assume that all agents contribute positively and in a timely manner (as depicted in Figure~\ref{fig:obs1}(a)).
However, our observations in Figure~\ref{fig:obs1}(b) uncover substantial differences in graph federated learning, even though all the participants are honest: 
(1) Specific agents, like agent 4 and 5, negatively impact the entire federation;
(2) The contributions of certain agents exhibit delays; for instance, agent 1 initially poses a harmful impact but contributes the most after multiple rounds of training. \emph{As the first contribution, we unveil a unique phenomenon in the context of graph federated learning: the presence of agents posing potentially harm and contributing with delays.}

\begin{figure*}[t]
	\centering     
    \vspace{-0.1in}

 \includegraphics[width=0.9\textwidth]{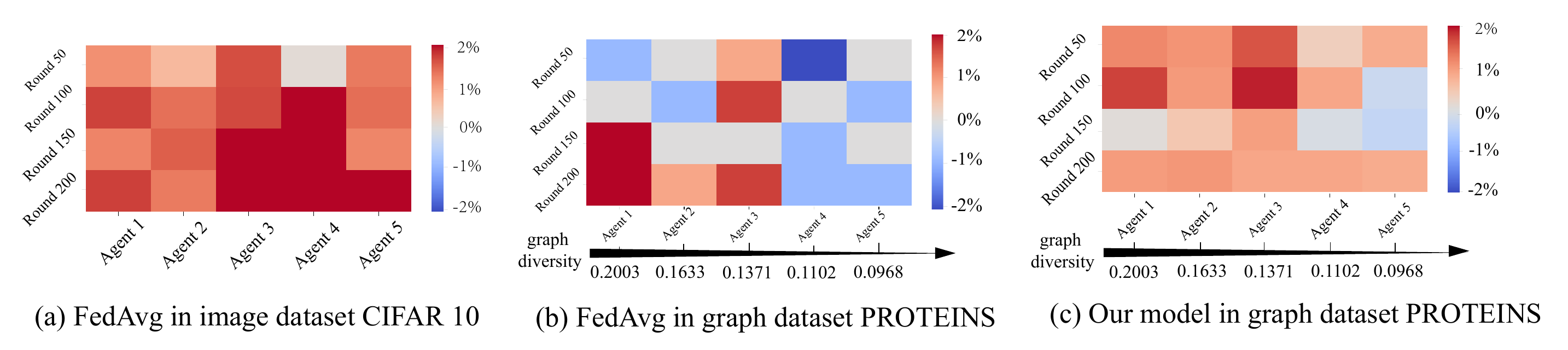} \\
    \vspace{-0.15in}
    \captionsetup{font={small}}
	\caption{The contribution of each agent in different communication rounds. Positive contributions are highlighted in red, while negative contributions are indicated in blue. The contribution of agent $i$ in $t$-th round is computed as the relative improvement in performance on the server's global test set in $t$-th round when all agents participate compared to when agent $i$ is excluded. {It is observed that in Figure 1(b) of graph federated learning, certain agents make negative contributions, while in Figure 1(a) of image federated learning, all agents make positive contributions. Additionally, some agents in graph federated learning (Figure 1(b)) initially pose a harmful impact but contribute positively to the federation after multiple rounds of training.}}

 \vspace{-0.1in}
	\label{fig:obs1}     
\end{figure*}

In light of the uniqueness of graph federated learning, we believe that an ideal incentive mechanism  for graph federated learning should simultaneously satisfy two criteria: the capability of (1) \emph{rewarding contributing agents while penalizing those causing harm}, and (2) \emph{offering post-hoc compensation to agents with delayed contributions.}

However, the primary challenge lies in  designing a comprehensive incentive mechanism framework that satisfies both criteria. Existing incentive mechanisms mainly fall into two lines, each with inherent limitations to the desired one. One line aims to allocate model gradients---the aggregated parameter updates/gradients that agents download from the server---to reward contributing agents \cite{xu2020reputation,xu2021gradient}. Yet, they cannot  penalize agents that may potentially do harm and compensate those with delayed contributions. 
The other line focuses on the allocation of payoff (such as monetary or computational resources) \cite{Gao2021FIFLAF,Yu2020AFI}, but fails to provide training-time incentives during training, potentially diminishing agent motivation.

To address this challenge, we put forth a novel incentive mechanism tailored for fair graph federated learning, outlined in Table~\ref{tab:incentive}.  This mechanism seamlessly integrates the allocation of both \emph{model gradients} and \emph{payoff}, whereby model gradients can function as rewards for contributing agents and offer acknowledgment in a training-time manner, while  payoff serve a dual purpose by not only imposing penalties on agents who may potentially do harm but compensating agents with delayed contributions.

\begin{table}[t]
\setlength\tabcolsep{3pt}
\centering
\vspace{-0.05in}

\begin{tabular}{ccc}
\hline
\multicolumn{1}{l}{}   & \makecell{\emph{Feature 1}:\\ reinforcement type}   & \multicolumn{1}{l}{\makecell{\emph{Feature 2}:\\ workflow phase}} \\ \hline

model gradient  & reward               & training-time                      \\
payoff               & reward \& punishment & post-hoc                           \\ \hline
\end{tabular}
\vspace{-0.05in}
\caption{Features of our proposed incentive mechanism  for fair graph federated learning.}
\label{tab:incentive}
\vspace{-0.25in}
\end{table}

Within this framework, the subsequent challenge is how to value an agent's contribution/harm to graph federation, such that the allocation can be conducted according to the agent value. 
Previous works typically value agents on an auxiliary validation set~\cite{jia2019towards, song2019profit, wang2020principled}, but selecting a validation set that is accessible and agreed upon by all agents poses challenges.
To address this issue, we introduce an agent valuation function, incorporating two criteria: \emph{gradient alignment} and \emph{graph diversity}. 
Gradient alignment exploit the similarity between local gradients and the server's global gradients. However, relying solely on gradient alignment could underestimate agents with delayed contributions. This is evident in Figure~\ref{fig:obs1}(b), where agents with delayed contributions often exhibit high diversity. We therefore introduce  an additional criterion: the graph diversity of agents' local data.

The remaining challenge lies in how to enhance the accuracy-fairness trade-off; this is critical especially in the presence of agents potentially do harm in graph federated learning.
Previous endeavors have attempted to aggregate the local gradients based on the agent value \cite{mcmahan2017communication,xu2021gradient}. However, 
this approach falls short of ensuring the model quality in graph federated learning, due to the high heterogeneity in graph data \cite{xie2021federated}. This paper introduces a novel concept of \emph{motif prototypes} as a reference coordinate  between server and agents,  facilitating not only the server's role in global model aggregation but also agents in optimizing their local models.

Our major contributions can be summairzed as follows:
\begin{itemize}[leftmargin=*,topsep=0pt]
\item \textbf{Problem.} To the best of our knowledge, we are \emph{the first} work to study the incentive mechanism for graph federated learning.
We unveil a unique phenomenon in the context of graph federated learning: the presence of agents posing potentially harm and contributing with delays.

\item \textbf{Method.} We propose a novel incentive mechanism tailored for fair graph federated learning that provides both model gradients and payoff for agents. Particularly, we propose to value the agents based on gradient alignment and graph diversity, and introduce motif prototypes to enhance accuracy-fairness trade-off.  

\item \textbf{Experiment.} 
In-depth experiments conducted across various settings reveal that our model not only demonstrates superiority in the allocation of model gradients and payoff, but also achieves the most favorable trade-off between fairness and accuracy among state-of-the-art baselines.

\end{itemize}

\vspace{-0.1in}
\section{Preliminary}
\vspace{-0.05in}
\vpara{Vanilla graph federated learning.}
Vanilla graph federated learning
involves $N$ \emph{honest} agents, where each agent $i$ holds a local graph dataset denoted as $D_i$, comprising a set of graphs. The objective is to learn a shared global model, typically a graph neural network (GNN),  across all clients, which can be formulated as 
\begin{equation}
\min_{(\omega_i, \omega_2\cdots,\omega_N)} \sum_{i=1}^{N}\frac{\left|D_i\right|}{\left|D\right|} L(\omega_i; D_i),
\end{equation}
where $\omega_1=\omega_2=\cdots =\omega_N$, $\omega_i$ represents the model parameters of agent $i$, {$L(\omega_i; D_i)$ is the loss function of graph classification for local training in agent $i$,} $\left|D_i\right|$ is the number of instances in client $i$, and $\left|D\right|$ is the total number of instances over all clients. 

This objective is achieved through a two-step process: aggregation and distribution.
In the aggregation step, the gradients uploaded by the individual agents, $\{\mathbf{u}_0^t, \cdots, \mathbf{u}_N^t\}$, are combined in the server to obtain the aggregated gradient $\mathbf{u}_{\mathcal{N}}^t$. This aggregation is performed using a weighted average, and the weights are proportional to the sizes of the local datasets:
\begin{equation}
\mathbf{u}_{\mathcal{N}}^t=\sum_{i=1}^N \frac{\left|D_i\right|}{\left|D\right|} \mathbf{u}_i^t.
\end{equation}

In the distribution step, the aggregated gradient $\mathbf{u}_{\mathcal{N}}^t$ is sent back to each agent, and each agent $i$ receives the same gradient $\mathbf{u}_{\mathcal{N}}^t$ as a reward in $t$-th communication round.

Our current focus is solely on addressing fairness concerns related to honest agents, without considering the examination of potential cheating behaviors among agents.

\vpara{Shapely value.}
In the context of FL, the Shapley value can be applied to assess the value of individual agents in the collaborative federation \cite{xu2021gradient,song2019profit}. It provides a way to quantify the contribution of each agent in improving the overall performance in federation.

The Shapley value, originally introduced in cooperative game theory~\cite{shapley1997value}, is a widely used concept for evaluating the contribution of individual players in a coalition game. It measures the expected marginal value that a player brings when joining different coalitions, considering all possible permutations of players.

\begin{definition}[Shapley Value]~\label{def:shapley}
    Let $\mathcal{N}$ denote the set of all agents (\emph{i.e.}, the grand coalition),  a coalition $\mathcal{S}\subseteq\mathcal{N}$ is the subset of $\mathcal{N}$, $\Pi_{\mathcal{N}}$ is the set of all possible permutation of $\mathcal{N}$.     For a given permutation $\pi \in \Pi_{\mathcal{N}}$, $\mathcal{S}_{\pi,i}$ represents the coalition of agents preceding agent $i$ in the permutation.
    The gradient-based Shapley value of agent $i \in \mathcal{N}$ is defined as
\begin{equation}
\label{shapley}
    \varphi_i := \frac{1}{N!}\sum_{\pi \in \Pi_{\mathcal{N}}}\left[\nu\left(\mathcal{S}_{\pi, i} \cup\{i\}\right)-\nu\left(\mathcal{S}_{\pi, i}\right)\right],
\end{equation}
where $\nu(\mathcal{S})$ represents the value function associated with coalition $\mathcal{S}$.
\end{definition}

\vspace{-0.1in}
\section{Methodology}

This section introduces our incentive mechanism framework for promoting fairness in graph federated learning; see Figure \ref{fig:framework} for an overview. We first provide an overview of our framework that encompasses the allocation of both model gradients and payoff in \S~\ref{subsec:fair}. To implement this framework, we tackle two pivotal questions: how to assess an agent's contribution or potential harm within the graph federation context in \S~\ref{subsec: reputation}, and how to ensure the accuracy of graph federated learning in \S~\ref{subsec: quality}.

\Hide{This section introduces our proposed incentive mechanism-based framework towards fair graph federated learning, which aims to enhance fairness of agents in graph federation while maintaining global model quality. 
In this section, we propose the incentive mechanism-based framework towards fair graph federated learning, which aims to enhance fairness of agents in graph federation while maintaining global model quality.
To achieve the goal, we need to address three key questions: (1) how can the model gradients and payoff mechanism ensure fairness in graph federation (\S~\ref{subsec:fair}); (2) how to value an agent's contribution/harm to graph federation (\S~\ref{subsec: reputation}), and (3) how to guarantee the global model quality, especially in the presence of harmful agents (\S~\ref{subsec: quality}).}

\subsection{Overview Framework} \label{subsec:fair}
Our goal is to enhance fairness in graph federation by combining model gradients and payoff allocation mechanisms (Table \ref{tab:incentive}) to reward contributing agents, penalize agents with potential harm, and provide post-hoc compensation for delayed contributions. 

{Before introducing the whole framework, we define the agent value, $r_i^t$, which indicates the contribution of agent $i$ in the $t$-th communication round (details about the agent value could be found in  \S~\ref{subsec: reputation}). We proceed to elaborate on these incentive mechanisms as below.}

\begin{figure}[t]
     \centering
    \includegraphics[width=1\columnwidth]{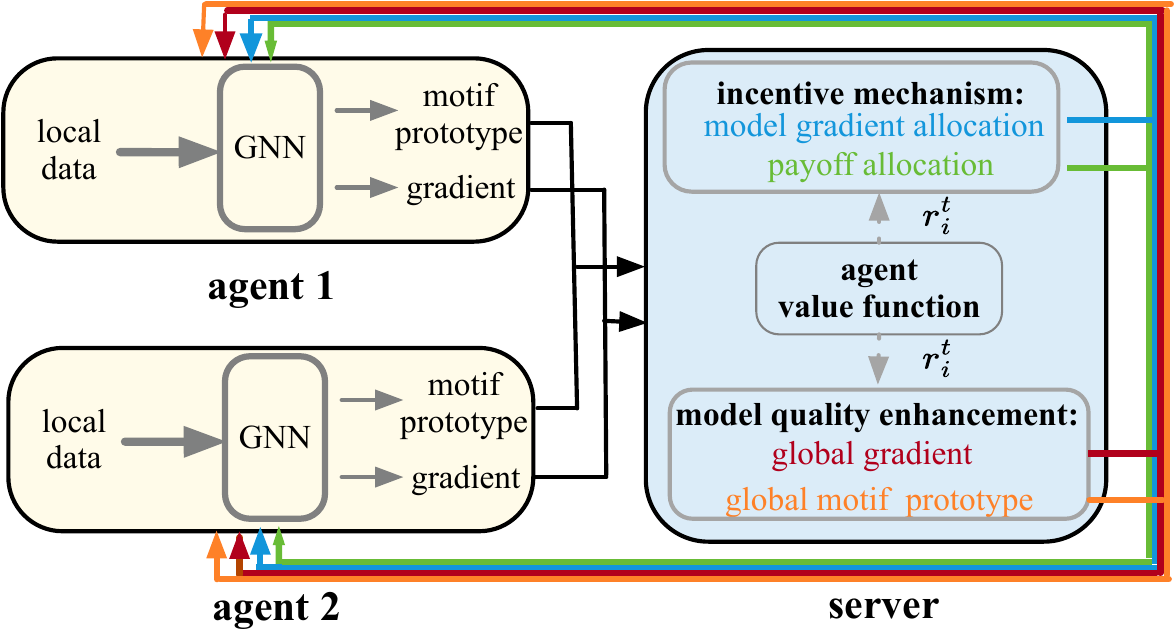}
      \vspace{-0.2in}
       \captionsetup{font=small}
    \caption{The proposed framework of fair graph federated learning.
    }
    \vspace{-0.3in}
    \label{fig:framework}
\end{figure}
\vpara{Model gradients allocation.}
In the vanilla FL framework, all agents download the same gradients from the server~\cite{mcmahan2017communication, li2020federated}, but this is unfair when dealing with agents with differing values/contributions. To address this issue, we propose that the server allocates global gradients to each agent based on their individual value, a mechanism termed model gradients allocation.

To achieve this, we employ a \emph{sparsifying gradient approach} for model gradients allocation, drawing inspiration from prior works like \cite{lyu2020collaborative, xu2020reputation, xu2021gradient}. This strategy entails rewarding agents who contribute more with denser gradient rewards, while those with less contributions receive sparser gradients. Agents who potentially do harm, in turn, are assigned zero gradients.

Specifically, to differentiate the quality of the gradients allocated to different agents, we can selectively sparsify the server's aggregated global gradient.
The sparsifying trick is achieved by a mask operation: when an agent's value is higher, we zero out fewer of the smaller components  of the global gradient, resulting in a higher-quality gradient reward.
The gradient  downloaded by agent $i$ in $t$-th communication round with value $r_{i}^t$  is
\begin{equation}
\text{gradient}_i^t=\text{mask}(\mathbf{u}_{\mathcal{N}}^t,\lfloor D \tanh (\beta r_i^t)/\max_{j\in \mathcal{N}}\tanh(\beta r_{j}^t) \rfloor),
\label{reward}
\end{equation}
where the operator $\text{mask}(\mathbf{u}_{\mathcal{N}}^t,x)$ returns the largest $\max(0,x)$ components of $\mathbf{u}_{\mathcal{N}}^t$, $D$ is the total number of components in the global gradient.
The hyper-parameter $\beta \geq 1$ controls the emphasis of fairness in FL: a smaller $\beta$ indicates a higher emphasis on fairness, as in this case agents with lower values will receive gradients of lower quality. In the extreme case of $\beta = \infty$, we revert to the vanilla FL.

Note that in scenarios where an agent offers no contribution or potentially do harm (\emph{i.e.}, $r_i^t \leq 0$), no gradient is allocated to it. This policy is adopted due to the impracticality of using model gradients to penalize harmful agents, for example, by applying inverse gradients, as doing so would contradict fundamental principles of federated learning.

\vpara{Payoff allocation.} 
To empower the framework's ability to penalize agents and compensate those with delayed contributions, we propose a scheme to allocate payoff  (\emph{e.g.}, money or computation resources).

On one hand, agents whose actions may have a negative impact on the federation, are subjected to penalties in payoff. For agent $i$ in the $t$-th round, if $r_i^t<0$, indicating that the agent is detrimental to the overall federation, we impose a payoff punishment $\text{payoff}_i^t=r_i^t < 0$ on the agent.

On the other hand, agents with delayed contributions  are provided payoff compensation  based on their performance history. 
This ensures that agents who have faced delays in their contributions are still acknowledged and rewarded accordingly.
By examining the agent values in previous rounds, we can estimate the agent's delayed contribution.
Specifically, we consider the agents, who have initially lower values but catch up in value over time, as those experiencing delayed contribution. To compensate,
we allocate the payoff compensation for agent $i$ in the $t$-th round (denoted as $\mu^t_i$) as the difference between the value in the current round and the average value of the previous rounds:
\begin{equation}
    \mu_i^t=\max(r_i^t - \frac{1}{t-1}\sum_{m=1}^{t-1} r_i^m,0).
\end{equation}

Given the agent value and the calculated payoff compensation, the payoff to agent $i$ in $t$-th communication round is 
\begin{align}
\begin{aligned}
    \text{payoff}_i^t=\left\{\begin{array}{ll}
r_i^t, &r_i^t < 0  \\
r_i^t+\mu_i^t, &\text {otherwise}
\end{array}\right.
\end{aligned},\\ \text{payoff}_i^t \leftarrow \frac{\text{payoff}_i^t}{\sum_{i=1}^N \text{payoff}_i^t}B,
\end{align}
where $B$ is the budget of payoff.

Given the overall framework above, two challenges remain to be solved: (1) how to define the agent value  $r_i^t$; (2) while the allocation mechanism contributes to fairness, it should also guarantee the model accuracy. We will address these two problems in the following two sections.

\vspace{-0.10in}
\subsection{Agent Valuation Function} \label{subsec: reputation}

The value of an agent is typically determined by its accuracy on an auxiliary validation set~\cite{jia2019towards, song2019profit, wang2020principled}. However, selecting a validation set that is accessible and agreed upon by all agents can be challenging. To decouple the agent valuation from validation, we introduce two criteria: \emph{gradient alignment} and \emph{graph diversity}.

\vpara{Gradient alignment.}
Most approaches that utilize the Shapley value in FL typically define the value function based on an auxiliary validation set that is shared and agreed upon by all agents. 
To overcome the challenge and inspired from~\cite{xu2020reputation, xu2021gradient}, we propose to utilize the gradient information as the value function  for computing the Shapley value, namely \emph{gradient-based Shapley value}, instead of using an auxiliary validation set.
The gradient-based Shapley value is defined as
that in Definition~\ref{def:shapley} by
assigning the value function $\nu$ to be
\begin{equation}    \nu(\mathcal{S})=\cos(\mathbf{u}_{\mathcal{S}},\mathbf{u}_{\mathcal{N}})=\left<\mathbf{u}_{\mathcal{S}},\mathbf{u}_{\mathcal{N}}\right>/(||\mathbf{u}_{\mathcal{S}}||\cdot||\mathbf{u}_{\mathcal{N}}||),
\end{equation}
where $\nu(\mathcal{S})$  is defined as the cosine similarity between the gradient of coalition $\mathcal{S}$ (\emph{i.e.}, local gradient from agent), denoted as $\mathbf{u}_{\mathcal{S}}$, and the gradient of the grand coalition $\mathcal{N}$  (\emph{i.e.}, global gradient in server), denoted as $\mathbf{u}_{\mathcal{N}}$. 
The gradient-based Shapley value measures the gradient contribution of each agent in the federation. 
A positive value indicates that the agent makes a greater contribution, as its gradient is positively aligned with the global gradient. A negative value implies that the agent's gradient is in the opposite direction of the global gradient. In this case, the agent is considered as potentially detrimental to the federation.

However, calculating the exact gradient-based Shapley value of an agent costs $\mathcal{O}(2^N)$. To address this issue, we find that cosine similarity between the agent's local gradient and the server's global gradient could be used as an approximation. 
The error between approximated value and exact value can be bounded, as illustrated in~\cite{xu2021gradient}.
Therefore, the gradient-based Shapley value of agent $i$ in  $t$-th communication round is approximated as
\begin{equation}
  {\varphi}_i^t  \approx \cos(\mathbf{u}_i^t, \mathbf{u}_{\mathcal{N}}^t),
    \label{equ:reputation}
\end{equation}
where $\mathbf{u}_i^t$ is the local gradient that agent $i$ uploads to the server in communication round $t$, and $\mathbf{u}_{\mathcal{N}}^t$ is the global gradient in communication round $t$.
The formal theorem and proofs can be found in Appendix \ref{subsec:thm}.

\vpara{Graph diversity.}
Merely considering gradient alignment is insufficient for two reasons.
First, agents with delayed contributions would face undervaluation if their assessment were solely based on gradient alignment. This occurs because agents equipped with diverse graph data hold substantial potential for delayed contributions, yet they may struggle to attain perfectly aligned gradients initially (as depicted in Figure~\ref{fig:obs1}(b)).
Second, diverse graphs encompass a wide range of structural patterns that can be universally shared among agents, leading to better generalization~\cite{tan2023federated}. 
Consequently, to provide a more comprehensive valuation of agents, we additionally introduce the criterion of graph diversity.

However, quantifying the diversity of graph data is not straightforward due to the complex structure. {Existing works use the number of subgraphs, data size or heuristic measures to quantify the diversity of graph \cite{yuan2021explainability,frasca2021understanding,xu2023better}, but overlook the inherent structural patterns within graph.} To address this issue, we propose the use of motifs as a means to represent the diversity of graph data.  Motifs are representative structural patterns in graphs that can reveal important information about underlying structure \cite{zhang2021motif}. Therefore, for each agent, we define graph diversity as the volume of motifs contained in agent's local graph. 

\begin{definition}[Graph Diversity]
    The graph diversity $d_i$ of agent $i$ is defined as the ratio of the number of unique motifs categories in the agent's local graph data $k_i$ to the number of unique motifs categories in the entire data contributed by all agents K, which can be formulated as $d_i=\frac{k_i}{K}$.
\end{definition}
Note that the motif information necessary for computing graph diversity (\emph{i.e.}, the motif prototypes elaborated in \S~\ref{subsec: quality}) will be communicated between the agents and the server, rather than transmitting the actual graph diversity values. 

Another advantage of introducing graph diversity is that it can help prevent agents from falling into the trap of converging to similar models. This is because, from the perspective of game theory, if only model gradients are allocated, agents might fall into this trap.

\vpara{Agent value updates.} Finally, \emph{the value of agent} is defined  by considering two aspects: (1) the incorporation of graph alignment and graph diversity, and (2) the consideration of both the current round's assessment and the historical assessment.
First, the graph alignment $\varphi_i^t$  and graph diversity $d_i$ are combined by weight.
Then, to better assess the agent value, both the historical assessment and the value on current round should be considered. Specifically, the server updates the agent value in round $t$ via a moving average of the current value and the historical value $r_i^{t-1}$. To ensure that the values of all agents sum up to 1 in $t$-th round, the server normalizes the agent values in the last step. In summary, the update of agent value could be decomposed as: 
\begin{equation}
    r_i^t= r_i^{t-1}+ \alpha_1 ({\varphi}_i^t + \alpha_2 d_i),
    \ r_i^t\leftarrow r_i^t/\sum_{j\in\mathcal{N}}r_{j}^t,
    \label{value}
\end{equation}
where {$\alpha_1$} can be viewed as a trade-off of the value from current round and the historical rounds, $\alpha_2$ denotes a trade-off parameter between graph alignment and graph diversity, and {$r_i^0=1/N$} is the initial value of the agent.

\vspace{-0.05in}

\subsection{ Model Quality Enhancement} \label{subsec: quality}

In graph federated learning, striking a balance between fairness and accuracy is of paramount importance, particularly when considering agents that have the potential to compromise the model performance  due to the inherent heterogeneity in graph data. To ensure the high-accuracy federated model, there is a need for establishing a reference coordinate that can be communicated between the server and the agents.
This is necessary for the server to convey its intentions and requirements to the agents, while offering agents  the opportunity to  refine their local models and strive for improvement.

This section first introduces a concept of motif prototype as such  a reference coordinate, then presents how the server and agents utilize  motif prototype for self enhancement.

\subsubsection{Motif prototypes.}
The underlying reason graph federated learning suffers from poor accuracy lies in the challenge of transferring knowledge across different agents' graph data. This difficulty is hard to address by simply aggregating the gradients of local models on the server, as did in most FL methods.
Considering these limitations, we introduce  {motifs}, which are sub-structures with rich structural information \cite{milo2002network}. As a specific type of motif contains similar structural information across graphs, motifs can effectively operate as transferable patterns across graphs, even when encountered in agents with varied data distributions.
With this  transferable pattern, we introduce the concept of motif prototypes, which facilitate the transfer of knowledge across agents, ultimately resulting in enhanced accuracy.
We formally define the motif prototypes as follows:

\begin{definition}[Motif Prototypes]
For agent $i$, suppose that the agent's local data involve  $K_i$ unique motifs. For the $k$-motif in agent $i$, we define the corresponding motif prototype as the mean of the embedding vectors of the graph instances containing $k$-th motif, \emph{i.e.},
\begin{equation}
    \mathbf{c}_{i,k}^t=\frac{1}{|D_{i,k}|}\sum_{G\in D_{i,k}}f_{\omega_i^t}(G),
    \label{local}
\end{equation}
where $D_{i,k}$ is the subset of $D_i$ that is comprised of graph instances that contain $k$-th motif, $f_{\omega_i^t}(G)$ and $\omega_i^t$ are the embedding of graph instance $G$ and the parameter of embedding layers of agent $i$ in $t$-th round, respectively.

\end{definition}

In each communication round,  except for model gradients, the motif prototypes are also communicated between the server and agents to ensure a high-accuracy federated model. 
Specifically,  the agents upload their local motif prototypes to the server. After receiving the local motif prototypes from the agents, the server then aggregates the local motif prototypes from all agents to obtain the global motif prototypes. Once the global motif prototypes are obtained, the server distributes them back to the agents.

It is noteworthy that the communication of prototypes would not entail much privacy leakage. This is because motif prototypes are 1D-vectors derived via the computation of average statistics from the low-dimensional representations of graph instances, which is an irreversible process~\cite{Michieli2021PrototypeGF,tan2022fedproto}.

\vpara{Value-based global model aggregation.}
To ensure the quality of the global model, we propose a value-based global model aggregation approach.  This approach aims to aggregate both the introduced motif prototypes and the gradients of agents based on their value $r_i^t$.
Specifically, we assign higher weights to agents with higher agent values, while excluding agents with negative values to prevent potential harm to the global model.

We first introduce the aggregation process for motif prototypes at the server.
Suppose that there are $K$ unique motifs in total. We aggregate the local motif prototypes from the agents based on their values $r_i^t$. Specifically, the global motif prototype of $k$-th motif in $t$-th round is defined as 
\begin{equation}    \mathbf{c}_{\mathcal{N},k}^t=\frac{\sum_{i\in\mathcal{N}_k}\text{ReLU}(r_i^t)\cdot \mathbf{c}_{i,k}^t}{\sum_{i\in\mathcal{N}_k}\text{ReLU}(r_i^t)},
    \label{global}
\end{equation}
where $\mathcal{N}_k$ denotes the set of agents that have motif $k$.

A similar strategy is employed when aggregating model gradients. 
The global gradient in the server for the $t$-th communication round, denoted as $\mathbf{u}_{\mathcal{N}}^t$, can be aggregated as 
\begin{equation}
    \mathbf{u}_{\mathcal{N}}^t = \frac{\sum_{i=1}^m\text{ReLU}(r_i^t)\cdot \mathbf{u}_i^t}{\sum_{i=1}^m\text{ReLU}(r_i^t)},
    \label{aggre}
\end{equation}
where the ReLU function plays a crucial role in the aggregation process by excluding agents with negative effects on the federation.

\vpara{Local model training.}
The global motif prototype serves as an instruction to guide the agents in updating their models in a desired direction.  This also  encourages non-rewarded agents to proactively identify and rectify their local issues before uploading gradients with the server. For example, agents with initially lower values  $r_i^t$ have the opportunity to increase their value by following the guidance provided by the global motif prototype.

In order to facilitate this process, we introduce a regularization term that encourages the local motif prototype $\mathbf{c}_{i,k}^t$ to approach the global motif prototype $\mathbf{c}_{\mathcal{N},k}^t$ associated with motif $k$. Accordingly, the local loss on agent $i$ can be formulated as follows:
\begin{equation}
    L(\omega_i, D_i) = {L_S(F(D_i), Y)}+\lambda \sum_k d(\mathbf{c}_{i,k}^t, \mathbf{c}_{\mathcal{N},k}^t),
    \label{loss}
\end{equation}
where 
$L_S$ is the supervised loss 
that measures the discrepancy between the model predictions $F(D_i)$ and the ground truth $Y$ on the local data of agent $i$, $\lambda$ is the trade-off parameter between the supervised loss $L_S$ and the motif prototypes-based  regularization, and $d$ is the L2 distance metric.

\vspace{-0.12in}
\section{Experiments}
\vspace{-0.05in}
In the experiments, we evaluate our model in graph classification task. We report the personalized accuracy and the accuracy of global model. Besides, we evaluate the model gradient fairness and payoff allocation fairness. Other results, such as ablation study of our model and parameter sensitivity, can be found in Appendix \ref{subsec:addexp}.
\vspace{-0.3cm}

\subsection{Experimental Setup}
\vspace{-0.1in}
\vpara{Datasets.} We use three graph classification datasets: PROTEINS\cite{Borgwardt2005ProteinFP}, DD \cite{dobson2003distinguishing}, and IMDB-BINARY \cite{Morris2020TUDatasetAC}. The first two are molecule datasets, while the last one is social network, each containing a set of graphs. 
We retain 10\% of all the graphs as the global test set for the server, and the remaining graphs are distributed to 10 agents. In each agent, we randomly split 90\% for training and 10\% for testing. 
In addition, for the payoff fairness, we introduce another setting using perturbed graph dataset. Specifically, for the graphs in each agent, we add varying ratio of perturbations to the structure. We designate 3 agents as low-quality, 3 agents as medium-quality, and 4 agents as high-quality by flipping ratios of $[0.7, 1)$, $[0.3, 0.7)$, and $[0, 0.3)$ of the total number of edges, respectively. 

\vpara{Baselines.} We compare with two kinds of baselines: FL methods and payoff allocation approaches. 
For FL methods, we compare against  (1) \textbf{Self-train}, where the agents train their models only on their local datasets; (2) two commonly used FL baselines \textbf{FedAvg} \cite{mcmahan2017communication} and \textbf{FedProx} \cite{li2020federated}, (3) \textbf{FedSage} \cite{zhang2021subgraph}, a graph FL framework that adopts FedAvg with GraphSage encoder \cite{hamilton2017inductive}; (4) \textbf{GCFL} and \textbf{GCFL+} \cite{xie2021federated}, two methods that utilize graph clustering to improve the effectiveness of graph FL; (5) \textbf{DW and EU} \cite{xu2021gradient}, incentive mechanisms that allocate model gradient using sparsifying gradient approach based on local data sizes and distance between local and global gradient respectively;
(6) \cite{xu2021gradient} and \textbf{CFFL} \cite{xu2020reputation}, incentive-based FL models in image domain (here we adapt them to our setting by substituting the encoder with the GNN used in our model). 

For the payoff allocation approaches, follow \cite{Gao2021FIFLAF}, we set the volume-based payoff function of agent $i$ as $\varPhi(i)=\log(1+\left|D_i\right|)$, and  compare against classical baselines such as (1) \textbf{equal incentive} \cite{yang2017designing}, where participants get equal payoff, \emph{i.e.}, $\varPhi(i)=1/N$; (2) \textbf{union incentive} \cite{gollapudi2017profit}, where the payoff proportion of agent $i$ is $\varPhi(\mathcal{N})-\varPhi(\mathcal{N}\backslash \{i\})$; (3) \textbf{individual incentive} \cite{yang2017designing}, where the payoff proportion of agent $i$ is $\varPhi(i)$; and (4) \textbf{Shapley incentive} \cite{shapley1997value}, where the payoff proportion of agent $i$ is served as Eq. \eqref{shapley}, where the value function $\nu$ is substituted with $\varPhi(i)$. The actual payoff for each agent can be calculated as $\dfrac{\varPhi(i)}{\sum_{i=1}^{N}\varPhi(i)}B$, where $B$ is the budget for federated learning.

\vpara{Metrics.} We evaluate our model's performance on the classification accuracy, as well as the fairness among agents. For graph classification accuracy, we consider personalized accuracy and global accuracy. Personalized accuracy is the average accuracy of the agents on the testing sets of their local graph.
Global accuracy is measured on the global test set.

To evaluate the fairness of our allocation, we consider two perspectives: model gradient and payoff. 
To measure the \textbf{fairness of model gradient}, we calculate the Pearson correlation coefficient $\rho(\xi,\psi)$ between the test accuracies $\xi$ achieved through self-training and that $\psi$ achieved by agents when collaborating in FL, following \cite{xu2021gradient}.  
To measure the \textbf{fairness of payoff}, the allocation approach could be assessed based on the agents with constructed low, medium, and high-quality data. If agents with high-quality data receive a larger payoff, it indicates a better payoff allocation approach.

\vpara{Implementation details.} 
Following the settings in \cite{xie2021federated}, we evaluate on the federated graph classification task. 
We set the parameters $\alpha_1$ and $\alpha_2$ in Eq. \eqref{value} as 0.05 and 1, the parameter $\lambda$ in Eq.~\eqref{loss} as 0.1,  $\beta$ in Eq. \eqref{reward} as 1, and  the budget $B$ of payoff as 1 (the budget $B$ in payoff allocation baselines is also set as 1).
We utilized a three-layer GIN network with a hidden size of 64 and a dropout rate of 0.5 for both the server and agent models. An Adam optimizer with a learning rate of 0.001 and weight decay of $5e^{-4}$ is employed.
The communication round is 200, the epoch of local training on agents is  1 and the batch size is 128. Besides, we utilized the motif extraction method in \cite{yu2022molecular}. See more details in Appendix \ref{subsec:implement}.
 Our codes are available at \url{https://github.com/zjunet/FairGraphFL}.

\vspace{-0.10in}

\begin{table*}[h]
\centering
\setlength{\tabcolsep}{2pt}
\renewcommand{\arraystretch}{1.2}
\normalsize
\vspace{-0.7em}
\resizebox{2\columnwidth}{!}{
\begin{tabular}{l|ccc|ccc|ccc}
\hline
\multicolumn{1}{l|}{} & \multicolumn{3}{c|}{PROTEINS}       & \multicolumn{3}{c|}{DD}                               & \multicolumn{3}{c}{IMDB-BINARY}                                                              \\ \cline{2-10} 
\multicolumn{1}{l|}{Dataset} & \makecell{model gradient\\ fairness} & \makecell{global\\ accuracy} & \multicolumn{1}{c|}{\makecell{personalized\\ accuracy}} &  \makecell{model gradient\\ fairness} & \makecell{global\\ accuracy} & \multicolumn{1}{c|}{\makecell{personalized\\ accuracy}}  & \makecell{model gradient\\ fairness} & \makecell{global\\ accuracy} & \multicolumn{1}{c}{\makecell{personalized\\ accuracy}}\\ \hline
\multicolumn{1}{l|}{Self-train}                  & \multicolumn{1}{c:}{-}      & \multicolumn{1}{c}{-}        &0.712$\pm$0.011          & \multicolumn{1}{c:}{-}             & \multicolumn{1}{c}{-}       & \multicolumn{1}{c|}{0.587$\pm$0.016}        &            \multicolumn{1}{c:}{-}                   &    \multicolumn{1}{c}{-}                         &   0.779$\pm$0.018                \\ \hline
\multicolumn{1}{l|}{FedAvg}                      &  \multicolumn{1}{c:}{0.700$\pm$0.153}    &  \multicolumn{1}{c}{0.748$\pm$0.006}     &    0.750$\pm$0.007      &   \multicolumn{1}{c:}{0.259$\pm$0.209}            & \multicolumn{1}{c}{0.668$\pm$0.032}       & 0.657$\pm$0.029        &   \multicolumn{1}{c:}{0.780$\pm$0.143}       & \multicolumn{1}{c}{0.779$\pm$0.007}            & 0.782$\pm$0.018                         \\
\multicolumn{1}{l|}{FedProx}                      & \multicolumn{1}{c:}{0.717$\pm$0.173} &  \multicolumn{1}{c}{0.731$\pm$0.017}      & 0.746$\pm$0.007           &   \multicolumn{1}{c:}{0.283$\pm$0.186}            & \multicolumn{1}{c}{0.633$\pm$0.025}       &   0.672$\pm$0.020      &  \multicolumn{1}{c:}{0.842$\pm$0.093}   &  \multicolumn{1}{c}{0.775$\pm$0.018}                          &    0.756$\pm$0.017                           \\ \hline
\multicolumn{1}{l|}{FedSage}                      & \multicolumn{1}{c:}{0.770$\pm$0.048}    &  \multicolumn{1}{c}{0.740$\pm$0.024} &  0.741$\pm$0.025         &   \multicolumn{1}{c:}{0.360$\pm$0.159}           & \multicolumn{1}{c}{0.671$\pm$0.021} & 0.663$\pm$0.023        &   \multicolumn{1}{c:}{0.870$\pm$0.086}      & \multicolumn{1}{c}{0.763$\pm$0.004}  &  0.768$\pm$0.003                           \\
\multicolumn{1}{l|}{GCFL}                    & \multicolumn{1}{c:}{0.725$\pm$0.185}    & \multicolumn{1}{c}{-}       &  0.772$\pm$0.019       &   \multicolumn{1}{c:}{0.439$\pm$0.094}           & \multicolumn{1}{c}{-}       & \textbf{0.698}$\pm$0.013        &   \multicolumn{1}{c:}{0.874$\pm$0.040}       & \multicolumn{1}{c}{-}             &  \textbf{0.830}$\pm$0.009                          \\
\multicolumn{1}{l|}{GCFL+}                    &  \multicolumn{1}{c:}{0.734$\pm$0.146}   &  \multicolumn{1}{c}{-}      &  \textbf{0.775}$\pm$0.019        &    \multicolumn{1}{c:}{0.379$\pm$0.190}          & \multicolumn{1}{c}{-}        & 0.692$\pm$0.013        &  \multicolumn{1}{c:}{0.825$\pm$0.094}      &  \multicolumn{1}{c}{-}            &   0.819$\pm$0.007 \\
\multicolumn{1}{l|}{FedStar}                    & \multicolumn{1}{c:}{0.763$\pm$0.043}   &  \multicolumn{1}{c}{-}      & 0.717$\pm$0.138         &   \multicolumn{1}{c:}{0.335$\pm$0.034}           & \multicolumn{1}{c}{-}        & 0.665$\pm$0.109        &    \multicolumn{1}{c:}{0.795$\pm$0.071}      &  \multicolumn{1}{c}{-}            &  0.776$\pm$0.086
\\ \hline

\multicolumn{1}{l|}{DW}       &  \multicolumn{1}{c:}{0.702$\pm$0.014}     & \multicolumn{1}{c}{0.723$\pm$0.014} &  0.730$\pm$0.014       &  \multicolumn{1}{c:}{0.305$\pm$0.326} & \multicolumn{1}{c}{0.672$\pm$0.031} &  0.669$\pm$0.033       &   \multicolumn{1}{c:}{0.723$\pm$0.089}       & \multicolumn{1}{c}{0.751$\pm$0.011}  &   0.762$\pm$0.009\\ 
\multicolumn{1}{l|}{EU}       &  \multicolumn{1}{c:}{0.733$\pm$0.016}    & \multicolumn{1}{c}{0.725$\pm$0.048} &  0.715$\pm$0.141       &  \multicolumn{1}{c:}{0.381$\pm$0.081}  & \multicolumn{1}{c}{0.668$\pm$0.010} & 0.658$\pm$0.382         &   \multicolumn{1}{c:}{0.747$\pm$0.010}      & \multicolumn{1}{c}{0.755$\pm$0.023}  &  0.759$\pm$0.069 \\
\multicolumn{1}{l|}{CFFL}       & \multicolumn{1}{c:}{0.690$\pm$0.117}    & \multicolumn{1}{c}{0.735$\pm$0.037} &  0.713$\pm$0.103       &  \multicolumn{1}{c:}{0.402$\pm$0.141}  & \multicolumn{1}{c}{0.658$\pm$0.053} & 0.658$\pm$0.112       &    \multicolumn{1}{c:}{0.795$\pm$0.100}      & \multicolumn{1}{c}{0.741$\pm$0.037}  &  0.752$\pm$0.034\\
\multicolumn{1}{l|}{\cite{xu2021gradient}}       &   \multicolumn{1}{c:}{0.750$\pm$0.040}  & \multicolumn{1}{c}{0.745$\pm$0.022} &  0.737$\pm$0.025         &  \multicolumn{1}{c:}{0.384$\pm$0.123}  & \multicolumn{1}{c}{0.682$\pm$0.011} &   0.659$\pm$0.013      &   \multicolumn{1}{c:}{0.796$\pm$0.085}       & \multicolumn{1}{c}{0.781$\pm$0.014}  &   0.778$\pm$0.019  \\
\hline
\multicolumn{1}{l|}{ours}                     &\multicolumn{1}{c:}{\textbf{0.787}$\pm$0.052}    & \multicolumn{1}{c}{\textbf{0.753}$\pm$0.018} &  0.751$\pm$0.017       &  \multicolumn{1}{c:}{\textbf{0.479}$\pm$0.067}      &   \multicolumn{1}{c}{\textbf{0.692}$\pm$0.017}    &   0.680$\pm$0.017      &   \multicolumn{1}{c:}{\textbf{0.907}$\pm$0.018}      &  \multicolumn{1}{c}{\textbf{0.801}$\pm$0.013}  &  0.794$\pm$0.005                          \\

\Hide{-reput                     & 0.750$\pm$0.015     & 0.760$\pm$0.015&  0.770$\pm$0.011        &  0.677$\pm$0.020            &        & 0.312$\pm$0.203        &   0.775$\pm$0.015       &             &  0.843$\pm$0.056                           \\
reput-motif                    & 0.711$\pm$0.012     & 0.723$\pm$0.013 &  0.767$\pm$0.021        &  0.663$\pm$0.032            &  & 0.258$\pm$0.233        &   \textbf{0.753}$\pm$0.010       &             &  0.803$\pm$0.104                           \\ 
}
\hline
\end{tabular}
}
\vspace{-0.1in}
\captionsetup{font=small}
\caption{Accuracy and fairness performance on three clean graph datasets. We report the average test accuracy of all agents (denoted as personalized), the accuracy of global models on the global test dataset (denoted as global) and the performance fairness. ``-'' means these methods do not have a single global model on the server.
}
\label{tab:res2}
\vspace{-0.75cm}
\end{table*}

\subsection{Experimental Results} \label{subsec:exp}

\vspace{-0.1in}
\vpara{Accuracy and model gradient fairness.} Table \ref{tab:res2} present the accuracy and model gradient fairness on clean and perturbed graph datasets, respectively.  Our model demonstrates the best performance in terms of model gradient fairness and global accuracy, indicating a significant advantage in fairness without compromising overall performance.
Traditional FL methods like FedAvg and FedProx do not perform well in  both  classification accuracy and model gradient fairness in most cases. This emphasizes the need of a dedicated design encompassing both the graph federated learning model and the graph incentive mechanism.
Graph FL methods like GCFL, GCFL+, FedSage and FedStar obtain either unsatisfactory accuracy and fairness, or good accuracy but with compromised fairness. This deficiency stems from the absence of an incentive mechanism within these models.
{The results of incentive-based federated learning methods such as DW, EU, \cite{xu2021gradient} and CFFL are also non-ideal, as they are not specifically designed  for graph FL.}
{We also note the presence of a trade-off between model gradient fairness and personalized accuracy, which has also been observed and recognized in previous research~\cite{gu2022privacy,ezzeldin2023fairfed}; in the subsequent experiments, we proceed to further evaluate this trade-off.}

\vspace{-0.05in}
\vpara{Trade-off between personalized accuracy and model gradient fairness.}
We evaluate the trade-off between model gradient fairness and personalized accuracy by presenting the results of various algorithms in Figure \ref{result}(a). The ideal algorithm should be situated in the upper right corner of the figure, similar to where our algorithm is positioned, signifying both high accuracy and excellent fairness. This results effectively demonstrate the most remarkable trade-off achieved by our algorithm between personalized accuracy and model gradient fairness.
\vspace{-0.05in}

\vpara{Payoff fairness.} To evaluate the payoff fairness of our model,  we compare it with  different payoff allocation approaches. We conduct experiments on perturbed datasets to examine the payoff received by low,  medium and high-quality agents.
The results are shown in Figure \ref{result}(b), where the height of the histogram is the average amount of the payoff received by each agent per communication round. It clearly indicates that our payoff allocation mechanism exhibits the best payoff fairness, 
as it provides the highest payoff to agents with high-quality data, while assigning negative payoff to agents with low-quality data. In contrast, other allocation methods struggle to maintain fairness in FL.
\vspace{-0.2cm}

\begin{figure}[t]
    \small
    \centering
    \includegraphics[width=0.5\textwidth]{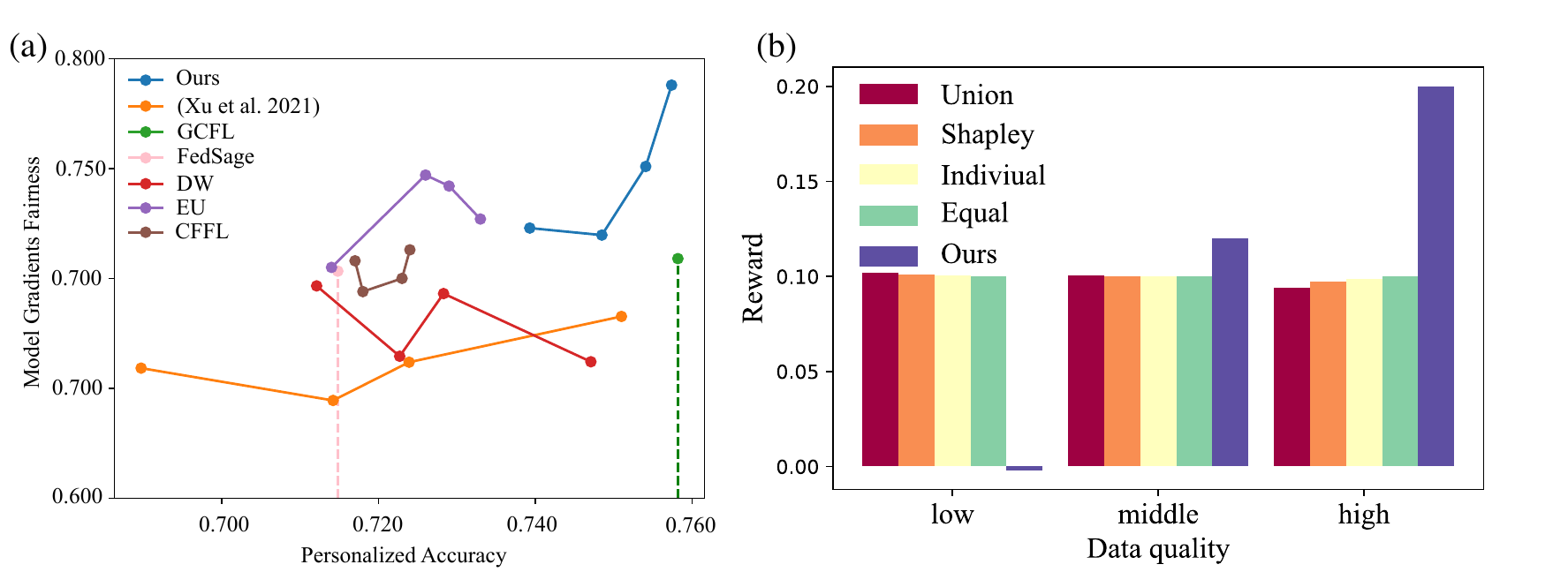}
    \vspace{-0.2in}
    \captionsetup{font=small}
    \caption{(a) The trade-off between model gradient fairness and personalized accuracy. (b) The average payoff per round for agents with varying data qualities.}
    \label{result}
    \vspace{-0.7cm}
\end{figure}

\Hide{
\begin{wraptable}{r}{9cm}
\vspace{0pt}
\centering
\begin{tabular}{c|ccc}
\hline
               & personalized & global & fairness \\ \hline
ours           & 0.751$\pm$0.017 & 0.753$\pm$0.018&0.787$\pm$0.052          \\
ours-value     &  0.757$\pm$0.014  & 0.759$\pm$0.014 &0.722$\pm$0.146    \\
ours-prototype & 0.737$\pm$0.025 & 0.745$\pm$0.022 & 0.750$\pm$0.040         \\
ours-diversity     & 0.744$\pm$0.009  & 0.748$\pm$0.022& 0.703$\pm$0.102 \\ \hline
\end{tabular}
\caption{Ablation study on PROTEINS.}
\label{ablation}
\vspace{-20pt}
\end{wraptable}

\vpara{Ablation Study.}
To demonstrate the effectiveness of each component within our model, we conducted ablation studies on three variants: ours-value, ours-prototype, and ours-diversity, which remove value assessment, prototype, and graph diversity in value assessment, respectively. The results are presented in Table \ref{ablation}. From the table, it can be observed that removing the prototype and motif components leads to a decrease in algorithm accuracy to some extent, and our algorithm experiences a significant decline in fairness. While the absence of the value assessment component slightly improves algorithm accuracy, it may not guarantee fairness. In summary, all parts of our algorithm are crucial in achieving both accuracy and fairness.
}

\section{Related Works}
\vspace{-0.05in}
\vpara{Incentive mechanism.} 
Data value in FL is primarily assessed from two dimensions:  data quantity and data quality. 
Concerning data quantity, existing works use data size to measure agent contribution \cite{zhan2020big,zhan2020learning}. As for data quality, most works evaluate the contribution of agents with the Shapley Value \cite{shapley1997value}, and a validation set on the server that is agreed by all agents is needed to determine the value function \cite{jia2019towards,song2019profit,wang2020principled}.

Built upon the notion of data value, incentive mechanisms are established for FL to allocate model gradients or payoff.
\cite{xu2021gradient,xu2020reputation}  allocate model gradients based on the similarity between local gradients and global gradients.
Another line of works introduce payoff-sharing schemes, where \cite{Yu2020AFI} dynamically allocates payoff by optimizing collective utility while minimizing inequality and  \cite{Gao2021FIFLAF} considers the malicious agents.
However, none of them have considered the uniqueness in graph federated learning.

\vpara{Graph federated learning.}
 Graph federated learning is divided into graph-level, subgraph-level and node-level
 ~\cite{he2021fedgraphnn}. In subgraph-level graph FL, 
 each agent has a subgraph of a large entire graph
 \cite{zhang2021subgraph,wu2021fedgnn}. In node-level graph FL, each agent possesses the ego-networks of one or multiple nodes \cite{meng2021cross,wang2022graphfl,zheng2021asfgnn}.
Our work belongs to graph-level graph FL, where each agent holds a set of graphs.  In this context, \cite{xie2021federated} introduces a clustered FL to deal with feature and structure heterogeneity.
\cite{tan2023federated} separates the graph structure and features, sharing only the structural information across agents while retaining features for local training by each agent.
\cite{gu2023dynamic} introduces a dynamic approach for 
 selecting clients and model gradients to enhance efficiency and accuracy.
\cite{he2021spreadgnn} proposes a multi-task
learning framework that eliminates the necessity for a central server.
However, none of them consider the fairness in graph federated learning.

 \vspace{-0.25cm}
\section{Conclusion}
\vspace{-0.03in}
In this paper, we begin by uncovering a unique phenomenon in graph federated learning: the presence of agents causing potential harm and agents contributing with delays. 
In light of this, we presents a novel incentive mechanism for fair graph federated learning framework, combining incentives from both model gradients and payoff. 
To achieve the framework, we introduce a agent valuation function considering both gradient alignment and graph diversity, and then we enhance the accuracy-fairness trade-off by introducing a novel concept of motif prototypes.
Extensive experiments demonstrate our superiority in both accuracy and fairness.

\section*{Acknowledgement}
This work is supported by NSFC (62206056, 12271479), Zhejiang NSF (LR22F020005) and the Fundamental Research Funds for the Central Universities.
\bibliography{aaai24}

\begin{thebibliography}{46}
\providecommand{\natexlab}[1]{#1}

\bibitem[{Borgwardt et~al.(2005)Borgwardt, Ong, Sch{\"o}nauer, Vishwanathan, Smola, and Kriegel}]{Borgwardt2005ProteinFP}
Borgwardt, K.~M.; Ong, C.~S.; Sch{\"o}nauer, S.; Vishwanathan, S. V.~N.; Smola, A.; and Kriegel, H.-P. 2005.
\newblock Protein function prediction via graph kernels.
\newblock \emph{Bioinformatics}, 21 Suppl 1: i47--56.

\bibitem[{Deng et~al.(2021)Deng, Lyu, Ren, Chen, Yang, Zhou, and Zhang}]{Deng2021FAIRQF}
Deng, Y.; Lyu, F.; Ren, J.; Chen, Y.-C.; Yang, P.; Zhou, Y.; and Zhang, Y. 2021.
\newblock FAIR: Quality-Aware Federated Learning with Precise User Incentive and Model Aggregation.
\newblock In \emph{IEEE Conference on Computer Communications}.

\bibitem[{Dobson and Doig(2003)}]{dobson2003distinguishing}
Dobson, P.~D.; and Doig, A.~J. 2003.
\newblock Distinguishing enzyme structures from non-enzymes without alignments.
\newblock \emph{Journal of molecular biology}, 330(4): 771--783.

\bibitem[{Ezzeldin et~al.(2023)Ezzeldin, Yan, He, Ferrara, and Avestimehr}]{ezzeldin2023fairfed}
Ezzeldin, Y.~H.; Yan, S.; He, C.; Ferrara, E.; and Avestimehr, A.~S. 2023.
\newblock Fairfed: Enabling group fairness in federated learning.
\newblock In \emph{Proceedings of the AAAI Conference on Artificial Intelligence}.

\bibitem[{Frasca et~al.(2021)Frasca, Bevilacqua, Bronstein, and Maron}]{frasca2021understanding}
Frasca, F.; Bevilacqua, B.; Bronstein, M.~M.; and Maron, H. 2021.
\newblock Understanding and Extending Subgraph GNNs by Rethinking Their Symmetries.
\newblock In \emph{Proceedings of the International Conference on Neural Information Processing Systems}.

\bibitem[{Gao et~al.(2021)Gao, Li, Chen, Zheng, Xu, and Xu}]{Gao2021FIFLAF}
Gao, L.; Li, L.; Chen, Y.; Zheng, W.; Xu, C.; and Xu, M. 2021.
\newblock FIFL: A Fair Incentive Mechanism for Federated Learning.
\newblock In \emph{Proceedings of the International Conference on Parallel Processing}.

\bibitem[{Gollapudi et~al.(2017)Gollapudi, Kollias, Panigrahi, and Pliatsika}]{gollapudi2017profit}
Gollapudi, S.; Kollias, K.; Panigrahi, D.; and Pliatsika, V. 2017.
\newblock Profit sharing and efficiency in utility games.
\newblock In \emph{Annual European Symposium on Algorithms}.

\bibitem[{Gu et~al.(2022)Gu, Tianqing, Li, Zhang, Ren, and Choo}]{gu2022privacy}
Gu, X.; Tianqing, Z.; Li, J.; Zhang, T.; Ren, W.; and Choo, K.-K.~R. 2022.
\newblock Privacy, accuracy, and model fairness trade-offs in federated learning.
\newblock \emph{Computers \& Security}, 122: 102907.

\bibitem[{Gu et~al.(2023)Gu, Zhang, Bai, Chen, Zhao, and Yang}]{gu2023dynamic}
Gu, Z.; Zhang, K.; Bai, G.; Chen, L.; Zhao, L.; and Yang, C. 2023.
\newblock Dynamic Activation of Clients and Parameters for Federated Learning over Heterogeneous Graphs.
\newblock In \emph{International Conference on Data Engineering}.

\bibitem[{Hamilton, Ying, and Leskovec(2017)}]{hamilton2017inductive}
Hamilton, W.~L.; Ying, Z.; and Leskovec, J. 2017.
\newblock Inductive Representation Learning on Large Graphs.
\newblock In \emph{Proceedings of the International Conference on Neural Information Processing Systems}.

\bibitem[{He et~al.(2021{\natexlab{a}})He, Balasubramanian, Ceyani, Yang, Xie, Sun, He, Yang, Yu, Rong et~al.}]{he2021fedgraphnn}
He, C.; Balasubramanian, K.; Ceyani, E.; Yang, C.; Xie, H.; Sun, L.; He, L.; Yang, L.; Yu, P.~S.; Rong, Y.; et~al. 2021{\natexlab{a}}.
\newblock Fedgraphnn: A federated learning system and benchmark for graph neural networks.
\newblock \emph{arXiv preprint arXiv:2104.07145}.

\bibitem[{He et~al.(2021{\natexlab{b}})He, Ceyani, Balasubramanian, Annavaram, and Avestimehr}]{he2021spreadgnn}
He, C.; Ceyani, E.; Balasubramanian, K.; Annavaram, M.; and Avestimehr, S. 2021{\natexlab{b}}.
\newblock Spreadgnn: Serverless multi-task federated learning for graph neural networks.
\newblock In \emph{Proceedings of the AAAI/ACM Conference on AI, Ethics, and Society}.

\bibitem[{Jia et~al.(2019)Jia, Dao, Wang, Hubis, Hynes, G{\"u}rel, Li, Zhang, Song, and Spanos}]{jia2019towards}
Jia, R.; Dao, D.; Wang, B.; Hubis, F.~A.; Hynes, N.; G{\"u}rel, N.~M.; Li, B.; Zhang, C.; Song, D.; and Spanos, C.~J. 2019.
\newblock Towards efficient data valuation based on the shapley value.
\newblock In \emph{International Conference on Artificial Intelligence and Statistics}.

\bibitem[{Karimireddy et~al.(2020)Karimireddy, Kale, Mohri, Reddi, Stich, and Suresh}]{karimireddy2020scaffold}
Karimireddy, S.~P.; Kale, S.; Mohri, M.; Reddi, S.; Stich, S.; and Suresh, A.~T. 2020.
\newblock Scaffold: Stochastic controlled averaging for federated learning.
\newblock In \emph{International Conference on Machine Learning}.

\bibitem[{Li et~al.(2020)Li, Sahu, Zaheer, Sanjabi, Talwalkar, and Smith}]{li2020federated}
Li, T.; Sahu, A.~K.; Zaheer, M.; Sanjabi, M.; Talwalkar, A.; and Smith, V. 2020.
\newblock Federated Optimization in Heterogeneous Networks.
\newblock In \emph{Proceedings of Machine Learning and Systems}.

\bibitem[{Lyu et~al.(2020)Lyu, Xu, Wang, and Yu}]{lyu2020collaborative}
Lyu, L.; Xu, X.; Wang, Q.; and Yu, H. 2020.
\newblock Collaborative fairness in federated learning.
\newblock \emph{Federated Learning: Privacy and Incentive}, 189--204.

\bibitem[{McMahan et~al.(2017)McMahan, Moore, Ramage, Hampson, and y~Arcas}]{mcmahan2017communication}
McMahan, B.; Moore, E.; Ramage, D.; Hampson, S.; and y~Arcas, B.~A. 2017.
\newblock Communication-efficient learning of deep networks from decentralized data.
\newblock In \emph{International Conference on Artificial intelligence and statistics}.

\bibitem[{Meng, Rambhatla, and Liu(2021)}]{meng2021cross}
Meng, C.; Rambhatla, S.; and Liu, Y. 2021.
\newblock Cross-node federated graph neural network for spatio-temporal data modeling.
\newblock In \emph{Proceedings of the ACM SIGKDD Conference on Knowledge Discovery and Data Mining}.

\bibitem[{Michieli and Ozay(2021)}]{Michieli2021PrototypeGF}
Michieli, U.; and Ozay, M. 2021.
\newblock Prototype Guided Federated Learning of Visual Feature Representations.
\newblock \emph{arXiv preprint arXiv:2105.08982}.

\bibitem[{Milo et~al.(2002)Milo, Shen-Orr, Itzkovitz, Kashtan, Chklovskii, and Alon}]{milo2002network}
Milo, R.; Shen-Orr, S.; Itzkovitz, S.; Kashtan, N.; Chklovskii, D.; and Alon, U. 2002.
\newblock Network motifs: simple building blocks of complex networks.
\newblock \emph{Science}, 298(5594): 824--827.

\bibitem[{Morris et~al.(2020)Morris, Kriege, Bause, Kersting, Mutzel, and Neumann}]{Morris2020TUDatasetAC}
Morris, C.; Kriege, N.~M.; Bause, F.; Kersting, K.; Mutzel, P.; and Neumann, M. 2020.
\newblock TUDataset: A collection of benchmark datasets for learning with graphs.
\newblock \emph{arXiv preprint arXiv:2007.08663}.

\bibitem[{Pan, Wu, and Zhu(2015)}]{Pan2015CogBoostBF}
Pan, S.; Wu, J.; and Zhu, X. 2015.
\newblock CogBoost: Boosting for Fast Cost-Sensitive Graph Classification.
\newblock \emph{IEEE Transactions on Knowledge and Data Engineering}, 27: 2933--2946.

\bibitem[{Shapley(1997)}]{shapley1997value}
Shapley, L.~S. 1997.
\newblock A value for n-person games.
\newblock \emph{Classics in game theory}, 69.

\bibitem[{Song, Tong, and Wei(2019)}]{song2019profit}
Song, T.; Tong, Y.; and Wei, S. 2019.
\newblock Profit allocation for federated learning.
\newblock In \emph{IEEE International Conference on Big Data}.

\bibitem[{Tan et~al.(2023)Tan, Liu, Long, Jiang, Lu, and Zhang}]{tan2023federated}
Tan, Y.; Liu, Y.; Long, G.; Jiang, J.; Lu, Q.; and Zhang, C. 2023.
\newblock Federated Learning on Non-IID Graphs via Structural Knowledge Sharing.
\newblock In \emph{Proceedings of the AAAI Conference on Artificial Intelligence}.

\bibitem[{Tan et~al.(2022)Tan, Long, Liu, Zhou, Lu, Jiang, and Zhang}]{tan2022fedproto}
Tan, Y.; Long, G.; Liu, L.; Zhou, T.; Lu, Q.; Jiang, J.; and Zhang, C. 2022.
\newblock Fedproto: Federated prototype learning across heterogeneous clients.
\newblock In \emph{Proceedings of the AAAI Conference on Artificial Intelligence}.

\bibitem[{Wang et~al.(2022)Wang, Li, Pang, Li, and Chen}]{wang2022graphfl}
Wang, B.; Li, A.; Pang, M.; Li, H.; and Chen, Y. 2022.
\newblock Graphfl: A federated learning framework for semi-supervised node classification on graphs.
\newblock In \emph{IEEE International Conference on Data Mining}.

\bibitem[{Wang et~al.(2020)Wang, Rausch, Zhang, Jia, and Song}]{wang2020principled}
Wang, T.; Rausch, J.; Zhang, C.; Jia, R.; and Song, D. 2020.
\newblock A principled approach to data valuation for federated learning.
\newblock \emph{Federated Learning: Privacy and Incentive}, 153--167.

\bibitem[{Wu et~al.(2021)Wu, Wu, Cao, Huang, and Xie}]{wu2021fedgnn}
Wu, C.; Wu, F.; Cao, Y.; Huang, Y.; and Xie, X. 2021.
\newblock Fedgnn: Federated graph neural network for privacy-preserving recommendation.
\newblock \emph{arXiv preprint arXiv:2102.04925}.

\bibitem[{Xie et~al.(2021)Xie, Ma, Xiong, and Yang}]{xie2021federated}
Xie, H.; Ma, J.; Xiong, L.; and Yang, C. 2021.
\newblock Federated graph classification over non-iid graphs.
\newblock In \emph{Proceedings of the International Conference on Neural Information Processing Systems}.

\bibitem[{Xie, Xiong, and Yang(2023)}]{xie2023federated}
Xie, H.; Xiong, L.; and Yang, C. 2023.
\newblock Federated node classification over graphs with latent link-type heterogeneity.
\newblock In \emph{Proceedings of the ACM Web Conference}.

\bibitem[{Xu et~al.(2023)Xu, Huang, Jiang, Cao, Yang, Wang, and Yang}]{xu2023better}
Xu, J.; Huang, R.; Jiang, X.; Cao, Y.; Yang, C.; Wang, C.; and Yang, Y. 2023.
\newblock Better with Less: A Data-Centric Prespective on Pre-Training Graph Neural Networks.
\newblock In \emph{Proceedings of the International Conference on Neural Information Processing Systems}.

\bibitem[{Xu and Lyu(2021)}]{xu2020reputation}
Xu, X.; and Lyu, L. 2021.
\newblock A Reputation Mechanism Is All You Need: Collaborative Fairness and Adversarial Robustness in Federated Learning.
\newblock \emph{ICML Workshop on Federated Learning for User Privacy and Data Confidentiality}.

\bibitem[{Xu et~al.(2021)Xu, Lyu, Ma, Miao, Foo, and Low}]{xu2021gradient}
Xu, X.; Lyu, L.; Ma, X.; Miao, C.; Foo, C.~S.; and Low, B. K.~H. 2021.
\newblock Gradient driven rewards to guarantee fairness in collaborative machine learning.
\newblock In \emph{Proceedings of the International Conference on Neural Information Processing Systems}.

\bibitem[{Yanardag and Vishwanathan(2015)}]{yanardag2015deep}
Yanardag, P.; and Vishwanathan, S. 2015.
\newblock Deep graph kernels.
\newblock In \emph{Proceedings of the ACM SIGKDD International Conference on Knowledge Discovery and Data Mining}.

\bibitem[{Yang et~al.(2017)Yang, Wu, Tang, Gao, Yang, and Chen}]{yang2017designing}
Yang, S.; Wu, F.; Tang, S.; Gao, X.; Yang, B.; and Chen, G. 2017.
\newblock On designing data quality-aware truth estimation and surplus sharing method for mobile crowdsensing.
\newblock \emph{IEEE Journal on Selected Areas in Communications}, 35(4): 832--847.

\bibitem[{Yu et~al.(2020)Yu, Liu, Liu, Chen, Cong, Weng, Niyato, and Yang}]{Yu2020AFI}
Yu, H.; Liu, Z.; Liu, Y.; Chen, T.; Cong, M.; Weng, X.; Niyato, D.~T.; and Yang, Q. 2020.
\newblock A Fairness-aware Incentive Scheme for Federated Learning.
\newblock \emph{Proceedings of the AAAI/ACM Conference on AI, Ethics, and Society}.

\bibitem[{Yu and Gao(2022)}]{yu2022molecular}
Yu, Z.; and Gao, H. 2022.
\newblock Molecular representation learning via heterogeneous motif graph neural networks.
\newblock In \emph{International Conference on Machine Learning}.

\bibitem[{Yuan et~al.(2021)Yuan, Yu, Wang, Li, and Ji}]{yuan2021explainability}
Yuan, H.; Yu, H.; Wang, J.; Li, K.; and Ji, S. 2021.
\newblock On Explainability of Graph Neural Networks via Subgraph Explorations.
\newblock In \emph{International Conference on Machine Learning}.

\bibitem[{Zhan et~al.(2020{\natexlab{a}})Zhan, Li, Qu, Zeng, and Guo}]{zhan2020learning}
Zhan, Y.; Li, P.; Qu, Z.; Zeng, D.; and Guo, S. 2020{\natexlab{a}}.
\newblock A learning-based incentive mechanism for federated learning.
\newblock \emph{IEEE Internet of Things Journal}, 7(7): 6360--6368.

\bibitem[{Zhan et~al.(2020{\natexlab{b}})Zhan, Li, Wang, Guo, and Xia}]{zhan2020big}
Zhan, Y.; Li, P.; Wang, K.; Guo, S.; and Xia, Y. 2020{\natexlab{b}}.
\newblock Big data analytics by crowdlearning: Architecture and mechanism design.
\newblock \emph{IEEE Network}, 34(3): 143--147.

\bibitem[{Zhang et~al.(2021{\natexlab{a}})Zhang, Zhang, Yu, and Yu}]{Zhang2021FASTGNNAT}
Zhang, C.; Zhang, S.; Yu, J. J.~Q.; and Yu, S. 2021{\natexlab{a}}.
\newblock FASTGNN: A Topological Information Protected Federated Learning Approach for Traffic Speed Forecasting.
\newblock \emph{IEEE Transactions on Industrial Informatics}, 17: 8464--8474.

\bibitem[{Zhang et~al.(2021{\natexlab{b}})Zhang, Yang, Li, Sun, and Yiu}]{zhang2021subgraph}
Zhang, K.; Yang, C.; Li, X.; Sun, L.; and Yiu, S.~M. 2021{\natexlab{b}}.
\newblock Subgraph federated learning with missing neighbor generation.
\newblock In \emph{Proceedings of the International Conference on Neural Information Processing Systems}.

\bibitem[{Zhang et~al.(2021{\natexlab{c}})Zhang, Liu, Wang, Lu, and Lee}]{zhang2021motif}
Zhang, Z.; Liu, Q.; Wang, H.; Lu, C.; and Lee, C.-K. 2021{\natexlab{c}}.
\newblock Motif-based graph self-supervised learning for molecular property prediction.
\newblock In \emph{Proceedings of the International Conference on Neural Information Processing Systems}.

\bibitem[{Zhao et~al.(2018)Zhao, Li, Lai, Suda, Civin, and Chandra}]{zhao2018federated}
Zhao, Y.; Li, M.; Lai, L.; Suda, N.; Civin, D.; and Chandra, V. 2018.
\newblock Federated learning with non-iid data.
\newblock \emph{arXiv preprint arXiv:1806.00582}.

\bibitem[{Zheng et~al.(2021)Zheng, Zhou, Chen, Wu, Wang, and Zhang}]{zheng2021asfgnn}
Zheng, L.; Zhou, J.; Chen, C.; Wu, B.; Wang, L.; and Zhang, B. 2021.
\newblock ASFGNN: Automated separated-federated graph neural network.
\newblock \emph{Peer-to-Peer Networking and Applications}, 14(3): 1692--1704.

\end{thebibliography}

\appendix
\onecolumn
\section{Appendix}
\subsection{Additional Theorems} \label{subsec:thm}
Let $\Pi_{\mathcal{N}}$ denote a set of all permutations of all possible permutations of $\mathcal{N}$. For a given permutation $\pi\in \Pi_{\mathcal{N}}$, $S_{\pi_i}$ represents the coalition of agents preceding agent $i$ in the permutation. Then the gradient-based Shapley value can be defined as follows:

\begin{align}
    \varphi_i :&= \frac{1}{N!}\sum_{\pi \in \Pi_{\mathcal{N}}}\left[\nu\left(\mathcal{S}_{\pi, i} \cup\{i\}\right)-\nu\left(\mathcal{S}_{\pi, i}\right)\right], \\ \nu(\mathcal{S})&=\cos(\mathbf{u}_{\mathcal{S}},\mathbf{u}_{\mathcal{N}})=\left<\mathbf{u}_{\mathcal{S}},\mathbf{u}_{\mathcal{N}}\right>/(||\mathbf{u}_{\mathcal{S}}||\cdot||\mathbf{u}_{\mathcal{N}}||),\notag
\end{align}
where $\mathbf{u}_{\mathcal{S}}$ denotes the gradient of coalition $\mathcal{S}$, and $\mathbf{u}_{\mathcal{N}}$ denotes the gradient of grand coalition $\mathcal{N}$. As the calculation process above needs $O(2^{|\mathcal{N}|})$ time complexity, we could use $\zeta_i=\cos(\mathbf{u}_i,\mathbf{u}_{\mathcal{N}})$ as an approximation , where $\mathbf{u}_i$ is the gradient that agent $i$ uploads in communication $t$. The approximation error theorem can be formulated as follows:
\begin{theorem}[Approximation Error]
    Let $I\in \mathbb{R}^+$. Suppose that $||\mathbf{u}_i^t||=\tau$ and $\left|\left<\mathbf{u}_i,\mathbf{u}_{\mathcal{N}}\right>\right|\leq 1/I$ for all $i\in \mathcal{N}$. Then $\varphi_i-L_i\zeta_i\leq I\tau^2$, where the multiplicative factor $L_i$ can be normalized.
\end{theorem}
The detailed proof can be found in \cite{xu2021gradient}.

\subsection{Additional Implementation Details}

\label{subsec:implement}
\vpara{Motif extraction methods.}
The motif prototype proposed above needs to define the motifs and extract the motifs from the graphs. Researchers have proposed many motif extraction methods and here we select a typical one that is proposed in \cite{yu2022molecular}. In this method, the bonds and rings are selected for the motifs. We first give a formal definition of the bonds and rings that is used as the motifs.
\begin{definition}[Bonds]
    A bond is a tuple that contains the edge information and the nodes information at the start and the end of the edge that does not exist on a ring structure. Specifically, given the edge $E$ on the graph, the bond can be denoted as $((E[0], E[1]), e)$, where $E[0]$ and $E[1]$ are the labels of the start node and the end node of edge $E$, respectively, and $e$ is the edge label of $E$. If $e$ does not contain the label, then $e$ is labeled as `None'.
\end{definition}

\begin{definition}[Rings]
    A ring is a tuple that contains the edge information and the node information that is contained in a cycle structure of a graph. Specifically, the ring can be formulated as the tuple $(V,E)$, where $V$ is the tuple of the labels of the nodes in the cycle structure and is ordered clockwise, $E$ is the tuple of the labels of the edges in the cycle structure and is ordered clockwise. If the edges do not contain the label, then it is labeled as `None'.
\end{definition}
In molecular graphs, the bonds would be viewed as the connection types between two atoms, such as the single bond and double bond between two carbon atoms. The rings could be viewed as function groups, such as benzene rings. In other graphs, the rings represent that the nodes on the rings are in the same community, and the bonds would imply that the two nodes have connections. Therefore, the motif extraction method can be used in all kinds of graph datasets.

\Hide{In molecular graphs, the labels of edges could represent the chemical bond types such as single bond and double bond, while in other kinds of graphs, the edges often represent that the two nodes have connections. Then the edges in social network is often labeled as `None'.}

Note that some of the motifs may appear in most of the graphs and they would carry little information in representation learning. To reduce the influence of these motifs, we use the term frequency-inverse document frequency (TF-IDF) algorithm to select the most essential motifs in our motif vocabulary. The TF-IDF value of a motif $k$ in the graph $G$ is computed as
\begin{equation}
    T_{k,G}=C(k)_G\left(\log\frac{1+|D_i|}{1+|D_{i,k}|}\right)+1,
\end{equation}
where $C(k)_G$ denotes the time that motif $k$ appears in graph $G$. To get the TF-IDF value $T_k$ for a motif $k$, we calculate the TF-IDF value for all graphs containing the motif $k$ as follows:
\begin{equation}
    T_k=\frac{1}{|D_{i,k}|}\sum_{k=1}^{|D_{i,k}|}T_{i,G_{n_k}},
\end{equation}
where the sequence $\{n_k\}$ denotes the index of graphs that contains the motif $k$. After sorting the motif vocabulary by the TF-IDF value, we select the essential motifs based the some threshold $\beta_s\in (0,1)$. In this work, we keep $\beta_s=0.9$. Then we keep only the most essential motifs as the final motif vocabulary.

\vpara{Trade-off experiment settings} For our algorithm, \cite{xu2021gradient}, DW, EU \cite{xu2021gradient} and CFFL \cite{xu2020reputation} that are five frameworks that are customized for graph federated learning, we set the hyper-parameter $\beta$ to be 0.5, 1, 1.5, 2, respectively, and record the personalized accuracy and model gradient fairness. For GCFL and FedSage, we set the same parameter as mentioned in \cite{xie2021federated}.

\vpara{Other implementation details.}Following the settings in \cite{xie2021federated}, we use 10 agents in the experiment to ensure that each agent would have a total of around 100 graph instances. Our model is implemented under the following software settings: Pytorch version 1.10.0+cu111, CUDA version 11.1, Python version 3.8.8, numpy version 1.20.3 and networkx version 2.8.7. All of the experiments are conducted on a single machine of Linux system with an Intel Xeon Gold 5118 (128G memory) and a GeForce Tesla P4 (8GB memory). The implementation can be found at \url{https://github.com/zjunet/FairGraphFL}.

\vpara{Dataset details.} We provide detailed information about the datasets that we utilized in Table \ref{stat}, and give a brief introduction to these datasets.

\begin{table*}[h] 
\centering
\begin{tabular}{cccccc}
\hline
\multirow{2}{*}{dataset} & \multicolumn{5}{c}{statistics}                                                                                                                                         \\ \cline{2-6} 
                         & \multicolumn{1}{l}{\#graphs} & \multicolumn{1}{l}{avg. \#nodes} & \multicolumn{1}{l}{avg. \#edges} & \multicolumn{1}{l}{\#classes} & \multicolumn{1}{l}{node features} \\ \hline
PROTEINS      &   1113          &  39.06               &     72.82       &     2          & original                          \\
DD            &    1178           & 284.32                 &    715.66           & 2           & original                          \\
IMDB-BINARY       &     1000     &   19.77            &   95.53                &   2      & degree                            \\ \hline
\end{tabular}
\caption{The statistics of the datasets.}
\label{stat}
\end{table*}
\begin{itemize}
    \item PROTEINS \cite{Borgwardt2005ProteinFP} is a protein dataset consisting of enzymes and non-enzymes. In this dataset, each node represents an amino acid, and an edge connects two nodes if their distance is less than 6 Angstroms. 
    \item DD \cite{dobson2003distinguishing} is an organic molecule dataset, where each node represents an atom and edges represent the chemical bonds between atoms. Nodes have different attributes, such as the type of atom, partial charge of atoms, etc. 
    \item IMDB-BINARY \cite{yanardag2015deep} is a movie collaboration dataset comprised of the ego-networks of 1,000 actors/actresses who have participated in movies in IMDB. In each graph, the nodes represent actors/actresses, and an edge exists between them if they have appeared in the same movie. 
\end{itemize}

\vpara{Details about Figure \ref{fig:obs1}.} In Figure \ref{fig:obs1}(a), we conduct FedAvg in image dataset CIFAR 10 by splitting the dataset into 5 equal parts. In Figure \ref{fig:obs1}(b) and (c), we conduct the experiment in graph dataset PROTEINS by splitting the dataset into 5 equal parts, using the FedAvg and our model, respectively. The contribution of each agent in different rounds is calculated as the relative improvement in performance on the server's global test set when all agents participate compared to when this agent is excluded.

\subsection{Additional Experiments} \label{subsec:addexp}
\vpara{Ablation study.}
We present the result of ablation study for our model to demonstrate the importance of each component in our model. 
\begin{table}[b]
\centering
\begin{tabular}{c|ccc}
\hline
               & \begin{tabular}[c]{@{}c@{}}personalized\\ accuracy\end{tabular} & \begin{tabular}[c]{@{}c@{}}global\\ accuracy\end{tabular} & \begin{tabular}[c]{@{}c@{}}model gradient\\ fairness\end{tabular} \\ \hline
ours           &  0.751$\pm$0.017                           &  0.753$\pm$0.018     &    0.787$\pm$0.052
\\
ours-gradient  &      0.743$\pm$0.017    & 0.750$\pm$0.019   &   0.736$\pm$0.145                     \\
ours-diversity &   \multicolumn{1}{l}{0.744$\pm$0.009} &0.748$\pm$0.022 &   0.723$\pm$0.082                                                                \\
ours-model gradient     &    \multicolumn{1}{l}{0.757$\pm$0.014} & 0.751$\pm$0.014  &  0.710$\pm$0.146                                                                \\
ours-aggregation     &    \multicolumn{1}{l}{0.741$\pm$0.042} & 0.740$\pm$0.024  &  0.741$\pm$0.121                                                                \\
ours-prototype &  \multicolumn{1}{l}{0.737$\pm$0.025} &0.741$\pm$0.022 &  0.750$\pm$0.040                    \\ \hline
\end{tabular}
\caption{Ablation study on accuracy and fairness performance on clean PROTEINS dataset.}
\label{ablation}
\end{table}
The five variants of our model remove the gradient alignment, graph diversity, model gradient, value-based aggregation and motif prototype component, respectively. The results can be found in Table~\ref{ablation}.
We can observe that:
(1) Comparing our model with $\textbf{ours-gradient}$ that removes the gradient alignment,  we can see that the inclusion of gradient alignment enhances both performance and model gradient fairness. The improvement can be attributed to the fact that gradient alignment allows for a more accurate estimation of agent value, which, in turn, contributes to the enhancement of global model quality through value-based aggregation on the server.
(2) Comparing our model with $\textbf{ours-diversity}$, which removes the graph diversity component,  it is implied that the inclusion of graph diversity could  bring the enhancement on both the model gradient fairness and the performance as it may lead to a more accurate estimation on the value of agents.  Notably, the graph diversity component has a more significant impact on model gradient fairness compared to gradient alignment, suggesting its crucial role in assessing agent value.
\Hide{
(2) $\textbf{ours-diversity}$ is a variant removing the graph diversity from our model. We could see that graph diversity is important in the value assessment of the agents and brings a larger impact to the model gradient fairness compared to gradient alignment. Besides, the two components, gradient alignment and graph diversity, could give us a more accurate assessment of the value of agents and thus bring less noise into the federation. Therefore, the removal of any component would lead to the drop in accuracy. }
(3)
Compared with $\textbf{ours-model gradient}$ which gives all the agents the same gradient, we could see that the model gradient fairness would be significantly hurt, as all agents receive the same model irrespective of their contributions. However, it leads to improved personalized accuracy, as all agents receive the same global gradient. This phenomenon could be viewed as a trade-off between the accuracy and the model gradient fairness.
\Hide{$\textbf{ours-model gradient}$ denotes the exclusion of model gradient allocation, which may result in a significant decline in model gradient fairness, as all agents receive the same model irrespective of their contributions. However, it leads to improved personalized accuracy, as all agents receive the same global gradient. This phenomenon could be viewed as a trade-off between the accuracy and the model gradient fairness.} (4) $\textbf{ours-aggregation}$ refers to the variant that aggregates local gradients in the server as FedAvg \cite{mcmahan2017communication}. The inclusion of our value-based aggregation may lead to an increase in accuracy and a better performance in model gradient fairness. This is because that our value-based aggregation technique reduces noise in the federation, thereby enhancing model accuracy. Additionally, it demonstrates relatively good performance in terms of model gradient fairness, as it considers both gradient alignment and graph diversity to assess the value of agents and thus allocate the model gradient relatively fairly.
(5) \textbf{ours-prototype} is a variant excluding the prototype-related components, and this could lead to a severe drop in accuracy due to the heterogeneity of the graph data. 
Moreover, the absence of the prototype-related components adversely affect the performance of model gradient fairness, as these components 
plays a crucial role in ensuring model quality, which in turn helps us assess the value of agents  more accurately through gradient alignment. 
\vpara{Hyper-parameter analysis.}
We did the hyper-parameter analysis on the trade-off parameter $\lambda$ between the supervised loss and the motif prototype-based regularization.
This is an important hyper-parameter in our proposed strategy. 
Figure~\ref{hyper} shows the effect of the trade-off parameter on global accuracy, personalized accuracy and model gradient fairness.
We can see that a too small or too large learning rate could deteriorate the accuracy and fairness performance, and the optimal performance can be obtained when $\lambda$ is 0.1.  
\begin{figure}[h]
    \small
    \centering
    \includegraphics[width=1.0\textwidth]{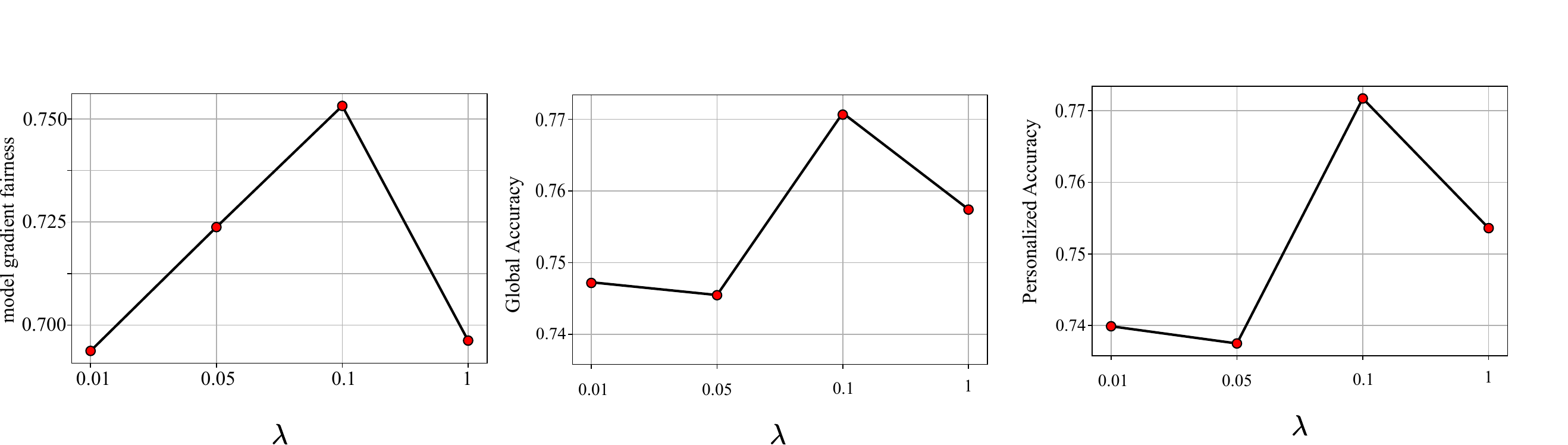}
    \vspace{-0.15in}
    \captionsetup{font=small}
    \caption{Personalized accuracy and model gradient fairness w.r.t. the trade-off parameter between the supervised loss and the motif prototype-based regularization.}
    \label{hyper}
    \vspace{-0.5cm}
\end{figure}

\vpara{Experiment on large dataset} We conducted the experiment on a large dataset, TWITTER-Real-Graph-Partial dataset
\cite{Pan2015CogBoostBF}, containing 144,033 graphs with an average of 4.03 nodes and 4.98 edges per graph. We compare with the most competitive Graph FL and incentive-based FL baselines.
The results in Table \ref{tab:res3} indicate our advantage even at a large scale.
\begin{table}[h]
\centering
\setlength{\tabcolsep}{1pt}
\renewcommand{\arraystretch}{1.2}
\vspace{-0.12in}
\resizebox{0.5\columnwidth}{!}{
\begin{tabular}{l|ccc}
\hline
\multicolumn{1}{l|}{} & \multicolumn{3}{c}{TWITTER-Real-Graph-Partial}      \\ \cline{2-4} 
\multicolumn{1}{l|}{Dataset} & \makecell{model gradient\\ fairness} & \makecell{global\\ accuracy} & \multicolumn{1}{c}{\makecell{personalized\\ accuracy}}\\ \hline
\multicolumn{1}{l|}{Self-train}                  & \multicolumn{1}{c:}{-}      & \multicolumn{1}{c}{-}        &0.594$\pm$0.002          \\ 
\multicolumn{1}{l|}{GCFL}                    & \multicolumn{1}{c:}{0.337$\pm$0.108}    & \multicolumn{1}{c}{-}       &  {0.612}$\pm$0.007      \\

\multicolumn{1}{l|}{(Xu et al. 2021)}       &   \multicolumn{1}{c:}{0.361$\pm$0.095}  & \multicolumn{1}{c}{0.601$\pm$0.010} &  0.598$\pm$0.003  \\

\multicolumn{1}{l|}{ours}                     &\multicolumn{1}{c:}{{0.401}$\pm$0.052}    & \multicolumn{1}{c}{{0.620}$\pm$0.014} &  0.600$\pm$0.008 \\
\hline
\end{tabular}
}
\captionsetup{font=small}
\caption{Results on a large scale dataset.}
\label{tab:res3}
\vspace{-0.3cm}
\end{table}

\subsection{Algorithm}
The pseudocode of our framework is summarized in Algorithm \ref{alg:1}. In each communication round, our framework mainly consists of three parts: (i) the server could aggregate the local gradients, update the value of agents and allocate the gradients to the agents according to their values (line 3-5). (ii) The agents then update the local models for $E$ rounds. (iii) The agents then update the local motif prototypes and the server update the global motif prototypes.

\begin{algorithm}[h] 
    \caption{Our federated learning framework}
    \begin{algorithmic}[1]
        \REQUIRE A total of $m$ agents, the number of local motifs $k_i,i=1,2,\cdots, m$, the number of motifs across the agents $K$, communication rounds $T$, the local iterations in a communication round $E$, the initial local models $\varPhi_i$, the local motif prototypes $\mathbf{c}^t_{i,k}, k=1,2,\cdots,K, t=1,2,\cdots,T$.
        \ENSURE The updated local models $\varPhi_i^T$.
        \STATE Initialize global prototypes for all motifs.
        \FOR{each round $t=1,2,\cdots, T$}
            \STATE Aggregate the local gradients by Eq. \eqref{aggre}.
            \STATE Update the value of agents by Eq. \eqref{value}.
            \STATE Allocate the gradient to agents by Eq. \eqref{reward}.
            \FOR{each client $i=1,2,\cdots m$ \textbf{in parallel}}
                \STATE Update the local model by Eq. \eqref{loss} for $E$ rounds.
            \ENDFOR
            \STATE Update the global motif prototypes by Eq. \eqref{local} and \eqref{global}.
        \ENDFOR
    \end{algorithmic}
    \label{alg:1}
\end{algorithm}

The time complexity of our model mainly consists of three parts. (1) The aggregation of local gradients and the update of the value costs $O(m)$. (2) The time complexity of local updates on each agent costs $O(E(X+k_i|D_i|))$ at most, where $X$ is the complexity of the backbone GNN. (3) The update of the global motif prototypes costs $O(Km)$ at most. Based on the analysis above, the overall time complexity of our framework on each agent in each communication round is $O((K+1)m+EX+Ek_i|D_i|)$ at most.

\subsection{Convergence Analysis of Our Algorithm}
In this subsection, we give the additional proof that our algorithm could finally convergent to a stationary point.
For the convenience in the detailed proof, set $\varPhi_i=(\omega_i,\phi_i)$, and decompose the parameters of the whole model $\varPhi_i$ into two parts: embedding parameter $\omega_i$ and the decision layer $\phi_i$. Then the local loss function of agent $i$ can be written as:
\begin{equation}
    L_i(\omega_i,\phi_i,D_i,Y)=L_S(F(D_i),Y)+\lambda \sum_{k=1}^{k_i}||\mathbf{c}_{i,k}^t-\mathbf{c}_{\mathcal{N},k}^t||,
\end{equation}
where the global motif prototype
\begin{equation}
\mathbf{c}_{\mathcal{N},k}^t=\frac{1}{|\mathcal{N}_k|}\sum_{i\in \mathcal{N}_k}\frac{|D_{i,k}|}{N_k}\mathbf{c}_{i,k}^t,
\end{equation}
and the local motif prototype 
\begin{equation}
    \mathbf{c}_{i,k}^t=\frac{1}{|D_{i,k}|}\sum_{G\in D_{i,k}}f_{\omega_i^t}(G).
\end{equation}

As for the iteration round notification, $e\in \{1/2,1,2\cdots,E\}$ to represent the local iterations. $tE$ represents the time step between prototype aggregation and gradient aggregation at the server and $tE+1/2$ represents the time step between the prototype aggregation at the server and the first iteration on the agents, and $L_{tE+e}$ means the loss function of the $e$-th local round in the communication round $t$. Then we could have the following theorem:
\begin{theorem}
    Suppose that the local loss function $L_i$ is $\mu$-smooth, the local embedding function $f_{\omega}$ is $L$-Lipschitz continuous, the stochastic gradient is an unbiased estimator of local gradient, the variance is bounded by $\sigma^2$, and the expectation of the local gradients is bounded by $G$. Then
    the loss function in an arbitrary agent strictly monotonically decreases in every communication round when
    \begin{equation}
        \lambda=\lambda_t=\frac{||\nabla_{tE+1/2}||}{2KLEG},
    \end{equation}
    and
    \begin{equation}
        \eta=\eta_{e'}=\frac{||\nabla_{tE+1/2}||^2}{\mu(E\sigma^2+\sum_{e=1/2}^{e'}||\nabla L_{tE+e'}||^2)}, e'=1/2,1,\cdots,E.
    \end{equation}
    As the loss function $L>0$, the loss function could converge according to the monotone bounded criterion.
\end{theorem}

To give a detailed proof, we first propose two lemmas that are important for the theorem. 
Let $\varPhi_{t+1}=\varPhi_t-\eta g_t$, we could have the following lemma:
\begin{lemma}
    From the beginning of communication round $t+1$ to the last local update step, the loss function of an arbitrary agent can be bounded as
    \begin{equation}
        \mathbb{E}[L_{tE+E}]\leq L_{tE+1/2}-(\eta-\frac{\mu\eta^2}{2})\sum_{e=1/2}^{E-1}||\nabla L_{tE+e}||^2+\frac{\mu E\eta^2}{2}\sigma^2.
    \end{equation}
\end{lemma}
Detailed proof can be found in \cite{tan2022fedproto}.
\begin{lemma}
    After the aggregation of local motif prototypes, the loss function of an agent can be bounded as:
    \begin{equation}
        \mathbb{E}[L_{(t+1)E+1/2}]< L_{(t+1)E}+\lambda KL\eta EG,
    \end{equation}
    where $K$ is the number of motifs on the server.
\end{lemma}
\begin{proof}
    \begin{align}
        L_{(t+1)E+1/2}&= L_{(t+1)E}+L_{(t+1)E+1/2}-L_{(t+1)E}\\
        &=L_{(t+1)E}+\lambda \sum_{k=1}^{k_i}\left(||\mathbf{c}_{i,k}^{t+1}-\mathbf{c}_{\mathcal{N},k}^{t+2}||-||\mathbf{c}_{i,k}^{t+1}-\mathbf{c}_{\mathcal{N},k}^{t+1}||\right)\\
        &\overset{(a)}{\leq} L_{(t+1)E}+\lambda \sum_{k=1}^{k_i}\left(||\mathbf{c}_{\mathcal{N},k}^{t+2}-\mathbf{c}_{\mathcal{N},k}^{t+1}||\right) \\
        &=L_{(t+1)E}+\lambda\sum_{k=1}^{k_i}\left(\frac{1}{|\mathcal{N}_k|}\sum_{i\in \mathcal{N}_k}\frac{|D_{i,k}|}{N_k}||\mathbf{c}_{i,k}^{t+2}-\mathbf{c}_{i,k}^{t+1}||\right)\\
        &\overset{(b)}{=}L_{(t+1)E}+
        \lambda \sum_{k=1}^{k_i}\left(\frac{1}{|\mathcal{N}_k|}\sum_{i\in\mathcal{N}_k}\frac{|D_{i,k}|}{N_k}\frac{1}{|D_{i,k}|}\sum_{G\in D_{i,k}}||f_{\omega_i^{t+2}}(G)-f_{\omega_i^{t+1}}(G)||\right)\\
        &\overset{}{\leq}L_{(t+1)E}+\lambda\sum_{k=1}^{k_i}\left(\frac{1}{|\mathcal{N}_k|}\sum_{i\in\mathcal{N}_k}\frac{L}{N_k}\sum_{G\in D_{i,k}}||\omega_i^{t+2}-\omega_i^{t+1}||\right)\\
        &\overset{(c)}{\leq} L_{(t+1)E}+\lambda\sum_{k=1}^{k_i}\left(\frac{1}{|\mathcal{N}_k|}\sum_{i\in\mathcal{N}_k}\frac{L}{N_k}\sum_{G\in D_{i,k}}||\varPhi_i^{t+2}-\varPhi_i^{t+1}||\right)\\
        &=L_{(t+1)E}+\lambda \sum_{k=1}^{k_i}\left(\frac{1}{|\mathcal{N}_k|}\sum_{i\in \mathcal{N}_k}\frac{L|D_{i,k}|}{N_k}\eta||\sum_{e=1/2}^{E-1}g_{i,tE+e}||\right)\\
        &\leq L_{(t+1)E}+\lambda \sum_{k=1}^{k_i}\left(\frac{1}{|\mathcal{N}_k|}\sum_{i\in\mathcal{N}_k}L\eta\sum_{e=1/2}^{E-1}||g_{i,tE+e}||\right)
    \end{align}
    After taking expectations from the random variable $\xi_t$ from both sizes of the equation, we could get:
    \begin{align}
        \mathbb{E}[L_{(t+1)E+1/2}]&\leq L_{(t+1)E}+\lambda \sum_{k=1}^{k_i}\left(\frac{1}{|\mathcal{N}_k|}\sum_{i\in\mathcal{N}_k}L\eta\sum_{e=1/2}^{E-1}\mathbb{E}[||g_{i,tE+e}||]\right)\\
        &\overset{}{\leq} L_{(t+1)E}+\lambda \sum_{k=1}^{k_i}\left(\frac{1}{\mathcal{N}_k}\sum_{i\in\mathcal{N}_k}L\eta EG\right)\\
        &=L_{(t+1)E}+\lambda k_iL\eta EG\leq L_{(t+1)E}+\lambda KL\eta EG ,
    \end{align}
    
    where (a) follows from the triangle inequality, (b) follows from Eq. \eqref{local}, (c) comes from the fact that $\omega_i^{t}$ is the subset of $\varPhi_i^t$.
\end{proof}

\vpara{Completing proof for Theorem 2.}
\begin{proof}
    From the above two lemmas, we could get that
\begin{equation}
    \mathbb{E}[L_{(t+1)E+1/2}]<L_{tE+1/2}-(\eta-\frac{\mu\eta^2}{2})\sum_{e=1/2}^{E-1}||\nabla L_{tE+e}||^2+\frac{\mu E\eta^2}{2}\sigma^2+\lambda KL\eta EG. 
\end{equation}
To ensure that the loss function decreases as the communication round $t$ increases, it is easy to get
\begin{equation}
    -(\eta-\frac{\mu\eta^2}{2})\sum_{e=1/2}^{E-1}||\nabla_{tE+e}||^2+\frac{\mu E\eta^2}{2}\sigma^2+\lambda KL\eta EG<0.
\end{equation}
Then the learning rate $\eta$ and the trade-off parameter $\lambda$ should satisfy:
\begin{equation}
    \eta<\frac{2\left(\sum_{e=1/2}^{E-1}||\nabla L_{tE+e}||^2-\lambda KLEG\right)}{\mu(E\sigma^2+\sum_{e=1/2}^{E-1}||\nabla L_{tE+e}||^2)}
\end{equation}

and
\begin{equation}
    \lambda<\frac{\sum_{e=1/2}^{E-1}||\nabla L_{tE+e}||^2}{KLEG}
\end{equation}
In practice, we cannot get the gradients after the current round. Therefore, we take the $\lambda$ to keep the same during one communication round:
\begin{equation}
    \lambda_t=\frac{||\nabla_{tE+1/2}||^2}{2KLEG},
\end{equation}
and keep the learning rate changing with the local update round:
\begin{equation}
    \eta_{e'}=\frac{||\nabla_{tE+1/2}||^2}{\mu(E\sigma^2+\sum_{e=1/2}^{e'}||\nabla L_{tE+e'}||^2)}, e'=1/2,1,2,\cdots, E.
\end{equation}

Here we complete the proof for Theorem 2.
\end{proof}

\clearpage


\end{document}